\documentclass[twoside,11pt]{article}

\usepackage{times}
\usepackage{fullpage}
\usepackage{graphicx} 
\usepackage{subfigure}
\usepackage{tikz}

\usepackage[round]{natbib}

\usepackage{algorithm}
\usepackage[noend]{algorithmic}
\usepackage{hyperref}

\usepackage{Definitions}
\usepackage{pdfsync}
\usepackage{color}
\usepackage{xspace}

\usepackage{enumitem}
\usepackage{comment}
\usepackage[font={small}]{caption}
\DeclareCaptionType{copyrightbox}

\newtheorem{claim}{Claim}

\newcommand{\Graph}{\Gcal}
\newcommand{\Node}{\Vcal}

\newcommand{\Edge}{\Ecal}

\newcommand{\Item}{\Lcal}
\newcommand{\Ground}{\Zcal}
\newcommand{\nGround}{N}
\newcommand{\Ind}{\Ical}

\newcommand{\budget}{b}
\newcommand{\attention}{u}

\newcommand{\Greedy}{G}
\newcommand{\g}{g}
\newcommand{\Optimal}{O}

\newcommand{\cur}{c}

\newcommand{\continmax}{\textsc{ConTinEst}\xspace}
\newcommand{\budgetmax}{\textsc{BudgetMax}\xspace}
\newcommand{\gdegree}{{GreedyDegree}\xspace}
\newcommand{\rmnum}[1]{\romannumeral #1}

\begin{document}

\title{Budgeted Influence Maximization for Multiple Products}

\author{
Nan Du, Yingyu Liang, Maria Florina Balcan, Le Song\\[0.1cm]
       {College of Computing, Georgia Institute of Technology}\\[0.1cm]
       \texttt{\{dunan,yliang39\}@gatech.edu, \{ninamf,lsong\}@cc.gatech.edu}
}

\date{}

\maketitle

\begin{abstract}

The typical algorithmic problem in viral marketing aims to identify a set of influential users in a social network, who, when convinced to adopt a product, shall influence other users in the network and trigger a large cascade of adoptions. However, the host (the owner of an online social platform) often faces more constraints than a single product, endless user attentions, unlimited budget and unbounded time; in reality, multiple products need to be advertised, each user can tolerate only a small number of recommendations, influencing user has a cost and advertisers have only limited budgets, and the adoptions need to be maximized within a short time window.

Given theses myriads of user, monetary, and timing constraints, it is extremely challenging for the host to design principled and efficient viral market algorithms with provable guarantees. In this paper, we provide a novel solution by formulating the problem as a submodular maximization in a continuous-time diffusion model under an intersection of a matroid and multiple knapsack constraints. We also propose an adaptive threshold greedy algorithm which can be faster than the traditional greedy algorithm with lazy evaluation, and scalable to networks with million of nodes. Furthermore, our mathematical formulation allows us to prove that the algorithm can achieve an approximation factor of $k_a/(2+2 k)$ when $k_a$ out of the $k$ knapsack constraints are active, which also improves over previous guarantees from combinatorial optimization literature. In the case when influencing each user has uniform cost, the approximation becomes even better to a factor of $1/3$. Extensive synthetic and real world experiments demonstrate that our budgeted influence maximization algorithm achieves the-state-of-the-art in terms of both effectiveness and scalability, often beating the next best by significant margins.

\end{abstract}

\section{Introduction} \label{sec:intro}

\setlength{\abovedisplayskip}{3pt}
\setlength{\abovedisplayshortskip}{3pt}
\setlength{\belowdisplayskip}{3pt}
\setlength{\belowdisplayshortskip}{2pt}
\setlength{\jot}{2pt}

\setlength{\textfloatsep}{1ex}

Online social networks play an important role in the promotion of new products, the spread of news, and the diffusion of technological innovations. In these contexts, the influence maximization problem (or viral marketing problem) typically has the following flavor: identify a set of influential users in a social network, who, when convinced to adopt a product, shall influence other users in the network and trigger a large cascade of adoptions. This problem has been studied extensively in the literature from both the modeling and algorithmic aspects~\citep{KemKleTar03,CheWanWan2010,BorBraChaLuc12,RodSch12,DuSonZhaGom13}.

However, the host (the owner of an online social platform) often faces more constraints than a single product, endless user attentions, unlimited budget and unbounded time; in reality
\begin{itemize}[noitemsep]
  \item {\bf Timing requirement:} the advertisers expect that the influence should occur within a certain time window, and different products may have different time requirements.
  \item {\bf Multiple products:} multiple products can spread simultaneously across the same set of social entities through different diffusion channels. These products may have different characteristics, such as revenue and speed of spread.
  \item {\bf User constraint:} users of the social network, each of which can be a potential source, would like to see only a small number of advertisement. Furthermore, users may be grouped according to their geographical locations and advertisers may have a target population they want to reach.
  \item {\bf Product constraint:} seeking initial adopters has a cost the advertiser needs to pay to the host, while the advertisers of each product have a limited amount of money.
\end{itemize}

Therefore, the goal of this paper is to solve the influence maximization problem by taking these myriads of practical and important constraints into consideration.

With respect to the multi-product and timing requirements, we propose to apply product-specific continuous-time diffusion models by incorporating the timing information into the influence estimation. Many previous work on influence maximization are mostly based on static graph structures and discrete-time diffusion models~\citep{KemKleTar03,CheWanWan2010,BorBraChaLuc12}, which cannot be easily extended to handle the asynchronous temporal information we observed in real world influence propagation. Artificially discretizing the timing information introduces additional tuning parameters, and will become more complicated in the multiple-product setting. A sequence of recent works argued that modeling cascade data and information diffusion using \emph{continuous-time} models can provide significantly improved performance than their discrete-time counterparts in recovering hidden diffusion networks and predicting the timing of events~\citep{DuSonSmoYua12, DuSonWooZha13, GomBalSch11, RodLesSch13, ZhoZhaSon13, ZhoZhaSon13b}. In our paper, we will also use the continuous-time diffusion models which provide us more accurate influence predictions~\citep{DuSonZhaGom13}.

With respect to the user and product constraints, we formulate these requirements by restricting the feasible domain over which the maximization is performed. We show that the overall influence function of multiple products is a submodular function, and the restrictions correspond to the constraints over the ground set of this submodular function. A very recent paper \citep{SomKakInaKaw14} studies the influence maximization subject to one knapsack constraint, but the problem is for one product over a known bipartite graph between marketing channels and potential customers, while we consider the more general and challenging problem for multiple products over general unknown diffusion networks. The work~\citep{LenBonCas10,SunCheLiuWanetal11} also seeks to select a fixed number of memes for each user so that the overall activity in the network is maximized. However, they have addressed the user constraints but disregarded the product constraints during the initial assignment. \citep{NarNan12} studies the cross-sell phenomenon (the selling of the first product raises the chance of selling the second), and the only constraint is a money budget for all the products. No user constraints are considered and the cost of assigning to different user is uniform for each product.

Finally, the recent work~\citep{LuBonAmiLak13} also considers the allocation problem of multiple items from the host's perspective, however with a few key differences from our work. First, \citep{LuBonAmiLak13} assumes that all items spread over the same fixed network structure given in advance based on the modified discrete-time diffusion model. Yet, in real scenarios, we may have no priori knowledge about the underlying network structure, and different items can have different diffusion structures as well, so we instead learn each product-specific diffusion networks directly from the data. Second, \citep{LuBonAmiLak13} considers the aspect of competition during the diffusion process without directly addressing the user and product constraints. In contrast, we model the constraints among multiple items during the initial stage of assignment due to users' dislike about advertisements and advertisers' budgets. Thirdly, \citep{LuBonAmiLak13} focuses the experimental evaluation of the proposed heuristic method only on the synthetic data. We instead provide mathematically rigorous formulation to design efficient algorithms with provable performance guarantee and further show in real testing data that the resulting allocation can indeed induce large scale diffusion.

Therefore, the main contributions of the paper include a novel formulation of a real world problem of significant practical interest, new efficient algorithms with provable theoretical guarantees, and strong empirical results. Furthermore,
\begin{itemize}[noitemsep]
	\item Unlike prior work that considers an a-priori described simplistic discrete-time diffusion model, we first {\bf learn} the diffusion networks from data by using continuous-time diffusion models. This allows us to address the timing constraints in a principled way.
	\item We formulate the influence maximization problem with aforementioned constraints as a submodular maximization under the intersection of matroid constraints and knapsack constraints. The submodular function we use is based on the actual diffusion model learned from data with the time window constraint. This novel formulation provides us a firm theoretical foundation for designing greedy algorithms with provable approximation guarantees.

	\item We propose an efficient adaptive-threshold greedy algorithm which is linear in the number of products and proportional to $\widetilde \Ocal(|\Node|+|\Edge^*|)$ where $|\Node|$ is the number of nodes (users) and $|\Edge^*|$ is the number of edges in the largest diffusion network. We prove that this algorithm is guaranteed to find a solution with an overall influence of at least roughly $\frac{k_a}{2+2k}$ of the optimal value, when $k_a$ out of the $k$ knapsack constraints are active. This improves over the best known approximation factor achieved by polynomial time algorithms in the combinatorial optimization literature. In the case when advertising each product to different users has uniform cost, the constraints reduce to an intersection of two matroids, and we obtain an approximation factor of roughly $1/3$, which is optimal for such optimization.
	\item We evaluate our algorithm over large synthetic and real world datasets. We observe that it can be faster than the traditional greedy algorithm with lazy evaluation, and is scalable to networks with millions of nodes. In terms of maximizing overall influence of all products, our algorithm can find an allocation that indeed induces the largest diffusion in the testing data with at least $20$-percent improvement overall compared to other scalable alternatives.
\end{itemize}
In the remainder of the paper, we first formalize our problem, modeling various types of practical requirements. We then describe our algorithm and provide the theoretical analysis. Finally, we present our experimental results and conclude the paper.

\section{Problem Formulation}
We will start with our strategies to tackle various types of practical requirements, and then describe the overall problem formulation.

\vspace{-2mm}
\subsection{Timing Constraints}
\vspace{-1mm}

The advertisers expect that the influence should occur within a certain time window, and different products may have different time requirements. To address this challenge, we will employ a continuous-time diffusion model which has been shown to perform better than discrete-time diffusion models in term of estimating diffusion influence given a time window~\citep{DuSonZhaGom13}.

More specifically, given a directed graph $\Gcal = (\Vcal,\Ecal)$, we associate each edge, $e:=(j,i)$, with a transmission function, $f_{e}(\tau_e)$. The transmission function is a density over time, in contrast to previous discrete-time models where each edge is associated with a fixed infection probability~\citep{KemKleTar03}. The diffusion process begins with a set of infected source nodes, $\Rcal$, initially adopting certain \emph{contagion} (idea, meme or product) at time zero. The contagion is transmitted from the sources along their out-going edges to their direct neighbors. Each transmission through an edge entails a \emph{random} transmission time, $\tau$, drawn independently from a density over time, $f_{e}(\tau)$. Then, the infected neighbors transmit the contagion to their respective neighbors. We assume an infected node remains infected for the entire diffusion process. Thus, if a node $i$ is infected by multiple neighbors, only the neighbor that first infects node $i$ will be the \emph{true parent}. The process continues until it passes an observation window $T$ or no more infection occurs.
This continuous-time independent cascade model lays a solid foundation for us to learn and describe the asynchronous temporal information of the cascade data. Specifically, by assuming particular parametric families~\citep{GomBalSch11,DuSonWooZha13} of the density function $f_{e}(\tau_e)$ or even the more sophisticated nonparametric techniques~\citep{DuSonSmoYua12}, we can learn the diffusion network structure as well as the density function $f_{e}(\tau_e)$ by using convex programming. Moreover, the learnt pairwise density function $f_{e}(\tau_e)$ can be sufficiently flexible to describe the heterogeneous and asynchronous temporal dynamics between pairs of nodes, which can be challenging for the classic discrete-time models to capture.

Intuitively, given a time window, the wider the spread of an infection, the more influential the given set of sources. The influence function is thus defined as the expected number of infected nodes given a set of sources by time $T$~\citep{RodSch12}. Formally, given a set, $\Rcal \subseteq \Vcal$, of sources infected at time zero and a time window $T$, a node $i$ is infected if $t_i\leqslant T$. The expected number of infected nodes (or {\bf the influence}) given the set of transmission functions $\cbr{f_{e}}_{(j,i)\in\Ecal}$ are defined as
\begin{align}
  \sigma(\Rcal,T)
    = \EE\left[\sum\nolimits_{i\in \Vcal}\II\cbr{t_i\leqslant T}\right],
  \label{eq:influence}
\end{align}
where $\II\cbr{\cdot}$ is the indicator function
and the expectation is taken over the the set of \emph{dependent}
variables $\{t_i\}_{i\in\Vcal}$. By Theorem 4 in~\citep{RodSch12}, the influence function $\sigma(\Rcal,T)$ is submodular in $\Rcal$.
In general, the exact influence estimation problem is a very challenging graphical model inference problem, so \citep{DuSonZhaGom13} has proposed a highly efficient randomized algorithm, \continmax for this task. It can estimate the influence of an arbitrary set of source nodes to an accuracy of $\epsilon$ using $r=\Ocal(1/\epsilon^2)$ randomizations and $\widetilde \Ocal(r|\Ecal|+r|\Vcal|)$ computations, so we will incorporate \continmax into our model.

\vspace{-2mm}
\subsection{Multiple Item Constraints}
\vspace{-1mm}

Multiple products can spread simultaneously across the same set of social entities through different diffusion channels. These products may have different characteristics, such as revenue and speed of spread. To address this challenge, we will use multiple diffusion networks for different types of products.

Suppose we have a set $\Lcal$ of different products that propagate on the same set of nodes $\Vcal$ with different diffusion dynamics. The diffusion network for product $i$ is denoted as $\Gcal_i = (\Vcal,\Ecal_i)$.
For each product $i$, we want to assign it to a set, $\Rcal_i\subseteq \Vcal$, of users (source nodes), while at the same time taking into account various constraints on the sets of source nodes.
Given a time $T_i$, let $\sigma_i(\Rcal_i, T_i)$ denote the influence of product $i$.

\begin{figure}[!t]
  \centering
  \includegraphics[width=.3\columnwidth]{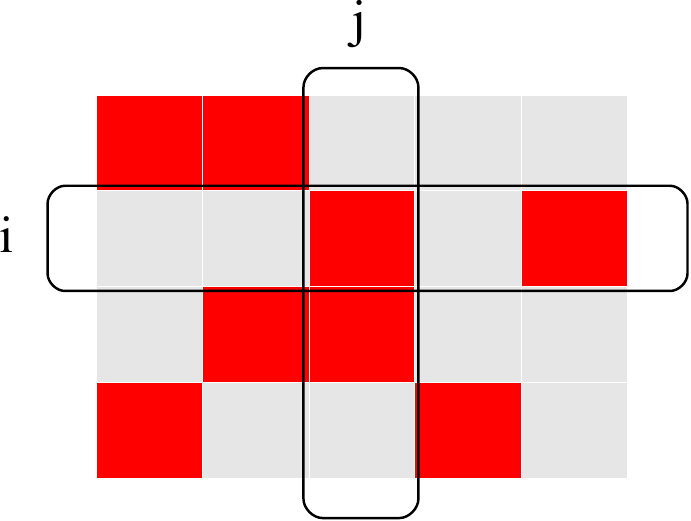}
  \caption{Illustration of the assignment matrix $A$ associated with partition matroid $\Mcal_1$ and group knapsack constraints. If product $i$ is assigned to user $j$, then $A_{ij} =1$ (colored in red).
  The ground set $\Ground$ is the set of indices of the entries in $A$, and selecting an element $(i,j) \in \Ground$ means assigning product $i$ to user $j$.
  The user constraint means that there are at most $u_j$ elements selected in the $j$-th column; the product constraint means that the total cost of the elements selected in the $i$-th row is at most $B_i$.
  } \label{fig:assignmentMatrix}
\end{figure}

The selection of $\Rcal_i$'s can be captured by an assignment matrix $A \in \{0,1\}^{|\Lcal| \times |\Vcal|}$ as follows: $A_{ij} = 1$ if $j \in \Rcal_i$ and $A_{ij} = 0$ otherwise.
Based on this observation, we define a new ground set $\Ground=\Item \times \Node$ of size $N=|\Item| \times |\Node|$. Each element of $\Zcal$ corresponds to the index $(i,j)$ of an entry in the assignment matrix $A$, and selecting element $z=(i,j)$ means assigning product $i$ to user $j$ (see Figure~\ref{fig:assignmentMatrix} for an illustration).
Then our goal is to maximize the {\bf overall influence}
\begin{align}
  f(S) := \sum_{i \in \Lcal} a_i \sigma_i(\Rcal_i, T_i)
\end{align}
subject to given constraints, where $a_i>0$ are a set of weights reflecting the different benefits of the products and $\Rcal_{i} = \{j \in \Node: (i,j) \in S\}$.
We now show that the overall influence function $f(S)$ is submodular over the ground set $\Ground$.

\begin{lemma}\label{lem:sub}
Under the continuous-time independent cascade model,
the overall influence $f(S)$ is a normalized monotone submodular function of $S$.
\end{lemma}
\begin{proof}
By definition, $f(\emptyset) = 0$ and $f(S)$ is monotone.
By Theorem 4 in~\citep{RodSch12}, the component influence function $\sigma_i(\Rcal_{i},T_i)$ for product $i$ is submodular in $\Rcal_{i} \subseteq \Node$.
Since non-negative linear combinations of submodular functions are still submodular, $f_i(S) := a_i\sigma_i(\Rcal_{i},T_i)$ is also submodular in $S \subseteq \Ground=\Item\times\Node$, and  $f(S) = \sum_{i\in \Item} f_i(S)$ is submodular.
\end{proof}

\vspace{-4mm}
\subsection{User Constraints}
\vspace{-1mm}

Users of the social network, each of which can be a potential source, would like to see only a small number of advertisement. Furthermore, users may be grouped according to their geographical locations and advertisers may have a target population they want to reach. To address this challenge, we will employ the matroids, a combinatorial structure that generalizes the notion of linear independence in matrices~\citep{Schrijver03,Fujishige05}.
Formulating our constrained influence maximization using matroids allow us to design a greedy algorithm with provable guarantees.

Formally, let each user $j$ can be assigned to at most $\attention_j$ products. Then
\begin{definition}
A matroid is a pair, $\Mcal=(\Ground, \Ind)$, defined over a finite set, $\Ground$ (the ground set), and
$\Ind$ contains a family of sets (the independent sets) which satisfy three axioms
\begin{enumerate}[noitemsep]
  \item {Non-emptiness:} The empty set $\emptyset \in \Ind$.
  \item {Heredity:} If $Y \in \Ind$ and $X\subseteq Y$, then $X \in \Ind$.
  \item {Exchange:} If $X \in \Ind, Y \in \Ind$ and $|Y| > |X|$, then there exists $z \in Y\setminus X$ such that $X \cup \{z\} \in \Ind$.
\end{enumerate}
\end{definition}

An important type of matroids are partition matroids in which the ground set $\Zcal$ is partitioned into disjoint subsets $\Ground_1, \Ground_2,\dots,\Ground_t$
for some $t$ and
$$\Ind=\{S~|~S\subseteq \Ground~\text{and}~|S \cap \Ground_i|\leqslant u_i, \forall i=1,\dots,t \}$$
for some given parameters $u_1,\dots, u_t$.
The user constraints can then be formulated as
\begin{itemize}
\item {Partition matroid $\Mcal_1$:} partition the ground set into $\Ground_{* j}=\Item \times\cbr{j}$ each of which corresponds to a column of  $A$. Then $\Mcal_1=\cbr{\Ground, \Ical_1}$ is
$$
	\Ical_1 = \cbr{S| S\subseteq \Ground~\text{and}~|S\cap \Ground_{* j}|\leqslant \attention_j,\forall j}.
$$
\end{itemize}

Note that matroids can model more general real world constraints that those described above,
and our formulation, algorithm, and theoretical results apply to general matroid constraints (more precisely, apply to Problem~\ref{pro:infMax}).  Our results can be used for significantly more general scenarios than the practical problem we addressed here.

For an concrete example, suppose there is a hierarchical community structure on the users, \ie, a tree $\Tcal$ whose leaves are the users and whose internal nodes are communities consisting of all users underneath, such as customers in different countries around the world. Due to policy or marketing strategies, on each community $C \in \Tcal$, there are at most $u_C$ slots for assigning the products.
Such constraints are readily modeled by the laminar matroid, which generalizes the partition matroid by allowing the subsets $\cbr{\Ground_i}$ to be a laminar family (\ie, for any $\Ground_i \neq \Ground_j$, either $\Ground_i \subseteq \Ground_j$, or $\Ground_j \subseteq \Ground_i$, or $\Ground_i \cap \Ground_j = \emptyset$). It can be verified that the community constraints can be captured by the matroid $\Mcal=(\Ground, \Ind)$ where $\Ical = \cbr{S\subseteq \Ground: |S\cap C|\leqslant u_C,\forall C \in \Tcal}$.

\vspace{-2mm}
\subsection{Product Constraints}
\vspace{-1mm}

Seeking initial adopters has a cost the advertiser needs to pay to the host, while the advertisers of each product have a limited amount of money. To address this challenge, we formulate the constraints as knapsack constraints.

Formally, let each product $i$ has a budget $B_i$ and assigning item $i$ to user $j$ costs $c_{ij} > 0$. For a set $\Item$ of products, the constraints correspond to $|\Item|$ group-knapsack constraints. To describe product constraints over the ground set $\Ground$, we introduce the following notations. For an element $z = (i,j) \in \Ground$, define its cost to be $c(z) := c_{ij}$. Abusing the notation slightly, the cost for a subset $S \subseteq \Ground$ is $c(S) := \sum_{z\in S} c(z)$. Then in a feasible solution $S \subseteq \Ground$, the cost of assigning product $i$ is $c(S \cap \Ground_{i*})$, which should not be larger than its budget $B_i$.

Without loss of generality, we can assume $B_i=1$ (by normalizing $c_{ij}$ with $B_i$),
and also $c_{ij} \in (0, 1]$ (by throwing away any element $(i,j)$ with $c_{ij} > 1$), and define
\begin{itemize}
	\item {Group-knapsack:} partition the ground set into $\Ground_{i *}=\cbr{i}\times \Node$ each of which corresponds to one row of $A$. Then a feasible solution $S \subseteq \Ground$ satisfies
	$$
		c(S \cap \Ground_{i*}) \leqslant 1, \forall i.
	$$
\end{itemize}
Note that these knapsack constraints have very specific structure: they are on different groups of a partition $\cbr{\Ground_{i *}}$ of the ground set. Furthermore, the submodular function $f(S) = \sum_i a_i\sigma_i(\Rcal_{i},T_i) $  are defined over the partition. Such structures allow us to design an efficient algorithm with improved guarantee over the known results.

\vspace{-2mm}
\subsection{Overall Problem Formulation} \label{sec:inf}
\vspace{-1mm}

Based on the above discussion of various constraints in viral marketing and our design choices for tackling the involved challenges, the influence maximization problem is a special case of the following constrained submodular maximization problem with $P=1$ matroid and $k=|\Item|$ knapsack constraints,
\begin{eqnarray}
& \text{max}_{S\subseteq \Ground} & f(S)  \label{pro:infMax}\\
& \text{subject to} &  c(S \cap \Ground_{i*}) \leqslant 1, \quad 1 \leqslant i \leqslant k, \nonumber\\
&  &  S\in \bigcap_{i=1}^P \Ical_p .\nonumber
\end{eqnarray}
For simplicity, let $\Fcal$ denote all the feasible solutions $S\subseteq \Ground$.

An important case of influence maximization, which we denote as {\bf Uniform Cost}, is that for each product $i$, different users have the same cost $c_{i*}$, \ie, $c_{ij}=c_{i*}$ for any $i$ and $j$. Equivalently, each product $i$ can be assigned to at most $b_j$ users, where $b_i:= \lfloor B_i / c_{i*} \rfloor$. Then the product constraints are simplified to
\begin{itemize}
 \item {Partition matroid $\Mcal_2$:} for the product constraints with uniform cost, define a matroid $\Mcal_2=\cbr{\Ground, \Ical_2}$ where $$\Ical_2 = \cbr{S | S\subseteq \Ground~\text{and}~|S\cap \Ground_{i*} |\leqslant \budget_i, \forall i}.$$
\end{itemize}
In this case, the influence maximization problem in Problem~\ref{pro:infMax} becomes one with $P=2$ matroid constraints and no knapsack constraints ($k=0$). It turns out that the analysis of this case (without knapsack) forms the base for that of the general case (with knapsack). In the following, we present our algorithm, and provide the analysis for the uniform cost case and then for the general case.

\section{Algorithm}\label{sec:algo}

For submodular maximization under multiple knapsack constraints,
there exist algorithms that can achieve $1-\frac{1}{e}$ approximation, but the running time is exponential in the number of knapsack constraints~\citep{KulShaTam09}. 
The matroid constraint in Problem~\ref{pro:infMax} can be replaced by $|\Node|$ knapsack constraints,
so that the problem becomes submodular maximization under $|\Item|+|\Node|$ knapsack constraints. However, this na\"ive approach is not practical for large scale scenarios due to the exponential time complexity.
For submodular maximization under $k$ knapsack constraints and $P$ matroids constraints,
the best approximation factor achieved by polynomial time algorithms is $\frac{1}{P+2 k + 1}$~\citep{BadVon13}.
This is not good enough, since in our problem $k=|\Item|$ can be large, though $P=1$ is small.
Note that Problem~\ref{pro:infMax} has very specific structure:
the knapsack constraints are over different groups $\Ground_{i*}$ of the whole ground set,
and the objective function is a sum of submodular functions over these different groups.
Here we exploit such structure to design an algorithm,
which achieves better approximation factor.

The details are described in Algorithm~\ref{alg:densityEnu}.
It enumerates different values of a so-called density threshold $\rho$, which quantifies the cost-effectiveness of assigning a particular product to a specific user.
It runs a subroutine to get a solutions for each $\rho$, 
and finally outputs the solution with maximum objective value.
Intuitively, the algorithm restricts the search space to be the set of most cost-effective allocations. 

The subroutine for a fixed density threshold is described in Algorithm~\ref{alg:greedyFixedDensity}. Inspired by the lazy evaluation heuristic, the algorithm maintains a working set $G$ and a marginal gain threshold $w_t$ geometrically decreasing by a factor of $1+\delta$, and sets the threshold to $0$ when it is sufficiently small. At each $w_t$, it selects new elements $z$ that satisfying the following:
(1) it is feasible and the density ratio (the ratio between the marginal gain and the cost) is over the current density threshold; (2) its marginal gain
$$
  f(z|\Greedy) := f(\Greedy \cup \{z\}) - f(\Greedy)
$$
is over the current marginal gain threshold.
The term ``density'' comes from the knapsack problem where the marginal gain is the mass and the cost is the volume, and large density means gaining a lot without paying much. In short, the algorithm considers only assignments with high quality, and repeatedly selects feasible ones with marginal gain from large to small.

\noindent{\bf Remark 1:} The traditional lazy evaluation heuristic also keeps a threshold but only uses the threshold to speed up selecting the element with maximum marginal gain.
Algorithm~\ref{alg:greedyFixedDensity} can add multiple elements $z$ from the ground set at each threshold,
and thus reduce the number of rounds from the size of the solution to the number of thresholds $\Ocal(\frac{1}{\delta}\log \frac{N}{\delta})$.
This allows us to tradeoff between the runtime and the approximation ratio (see our theoretical guarantees).

\noindent{\bf Remark 2:} Evaluating the objective $f$ is expensive, which involves evaluating the influence of the assigned products. We will use the randomized algorithm by~\citep{DuSonZhaGom13} to compute an estimation $\widehat f(\cdot)$ of the quantity $f(\cdot)$.

\renewcommand{\algorithmicrequire}{\textbf{Input:}}
\renewcommand{\algorithmicensure}{\textbf{Output:}}
\begin{algorithm}[!t]
\caption{Density Threshold Enumeration}
\label{alg:densityEnu}
\begin{algorithmic}[1]
\REQUIRE{parameter $\delta$; objective $f$ or its approximation $\widehat f$}
\STATE{Set $d = \max \cbr{ f(\{z\}): z \in \Ground}$.}
\FOR{ $\rho \in \cbr{\frac{2d}{P+2k+1}, (1+\delta)\frac{2d}{P+2k+1}, \dots, \frac{2|\Ground|d}{P+2k+1} }$ }
\STATE{Call Algorithm~\ref{alg:greedyFixedDensity} to get $S_\rho$.}
\ENDFOR
\ENSURE{$\argmax_{S_\rho} f(S_\rho)$.}
\end{algorithmic}
\end{algorithm}

\begin{algorithm}[!t]
\caption{Adaptive Threshold Greedy for Fixed Density}
\label{alg:greedyFixedDensity}
\begin{algorithmic}[1]
\REQUIRE{parameters $\rho$, $\delta$; objective $f$ or its approximation $\widehat f$}
\STATE{Set $d_\rho = \max \cbr{f(\cbr{z}): z \in \Ground, f(\cbr{z}) \geqslant c(z) \rho}$.\\
	Set $w_t = \frac{d_\rho}{(1+\delta)^t}$ for $t = 0,\dots, L= \argmin_i \bigl[w_i \leqslant \frac{\delta d}{\nGround}\bigr]$, and $ w_{L+1} = 0$.}
\STATE{Set $\Greedy= \emptyset$.}

\FOR{$t=0,1,\dots,L,L+1$}
    \FOR{$z \not\in \Greedy$ with $\Greedy \cup \{z\} \in \Fcal$
        and $f(z|\Greedy) \geqslant c(z) \rho$}
        \IF{$f(z|\Greedy) \geqslant w_t$}
	        \STATE{Set $\Greedy \leftarrow \Greedy \cup \{z\}$.}
        \ENDIF
    \ENDFOR
\ENDFOR

\ENSURE{$S_\rho=G$.}
\end{algorithmic}
\end{algorithm}

\section{Theoretical Guarantees} \label{sec:thm}

Our algorithm is simple and intuitive. However, it is highly non-trivial to obtain the theoretical guarantees. For clarity, we first analyze the simpler case with uniform cost, which then provides the base for analyzing the general case.

\subsection{Uniform Cost}\label{sec:uni}

As shown at the end of Section~\ref{sec:inf}, the influence maximization in this case corresponds to Problem~\ref{pro:infMax} with $P=2$ and no knapsack constraints. We can simply run Algorithm~\ref{alg:greedyFixedDensity} with $\rho = 0$ to obtain a solution $G$, which is then roughly $\frac{1}{P+1}$-approximation.

\noindent{\bf Intuition.} The algorithm greedily selects the feasible element with sufficiently large marginal gain. One might wonder whether the algorithm will select just a few elements while many elements in the optimal solution $O$
will become infeasible and will not be selected, in which case the greedy solution $G$ is a poor approximation.
Furthermore, we only use the estimation $\widehat f$ of the influence $f$ (\ie, $|\widehat f(S) - f(S)| \leqslant \epsilon$ for any $S \subseteq \Ground$),
which introduces additional error to the function value.
A crucial question that has not been addressed is whether the adaptive threshold greedy algorithm is robust to such perturbations.

It turns out that the algorithm will select sufficiently many elements of high quality.
First, the elements selected in optimal solution $O$ but not selected in $G$ can be partitioned into $|G|$ groups, each of which associates with an element in $G$,
such that the number of elements in the groups associating with the first $t$ elements in $G$ is bounded by $Pt$.
See Figure~\ref{fig:greedyPartition} for an illustration.
Second, the marginal gain of each element in $G$ is at least as large as that of any element in the group associated with it (up to some small error).
This means that even if the submodular function evaluation is inexact,
the quality of the elements in the greedy solution is still good.
The two claims together show that the marginal gain of $O\setminus G$ is not much larger than the gain of $G$, and thus $G$ is a good approximation for the problem.

Formally, suppose we use an inexact evaluation such that $|\widehat f(S) - f(S)| \leqslant \epsilon$ for any $S \subseteq \Ground$, and suppose product $i \in \Item$ spreads according to diffusion network $\Graph_i = (\Node, \Edge_i)$, and let $i^*=\argmax_{i\in\Item}|\Edge_i|$. We have

\begin{theorem}\label{thm:infMax_uni}
For influence maximization with uniform cost,  Algorithm~\ref{alg:greedyFixedDensity} (with $\rho=0$) outputs a solution $G$ with
$
	f(\Greedy) \geqslant \frac{1-2\delta}{3} f(\Optimal)
$
in expected time $\widetilde\Ocal\left(\frac{|\Edge_{i^*}|+|\Node|}{\delta^2}  + \frac{|\Item||\Node|}{\delta^3} \right).$
\end{theorem}

The parameter $\delta$ introduces a tradeoff between the approximation guarantee and the runtime:
larger $\delta$ decreases the approximation ratio but needs fewer influence evaluations.
The running time has a linear dependence on the network size and the number of products to propagate (ignoring some small logarithmic terms),
so the algorithm is scalable to large networks.

\noindent{\bf Analysis.}
Suppose $G=\{g_1,\dots, g_{|G|}\}$ in the order of selection, and let $G^t =\{\g_1, \dots, \g_t\}$.
Let $C_t$ denote all those elements in $O \setminus G$ that satisfy the following:
they are still feasible before selecting the $t$-th element $g_t$ but are infeasible after selecting $g_t$.
That is, $C_t$ are all those elements $j\in O \setminus G$ such that:
(1) $j \cup G^{t-1}$ does not violate the matroid constraints but (2) $j \cup G^{t}$ violates the matroid constraints.
In other words, $C_t$ are the optimal elements ``blocked'' by $g_t$. See Figure~\ref{fig:greedyPartition} for an illustration. 

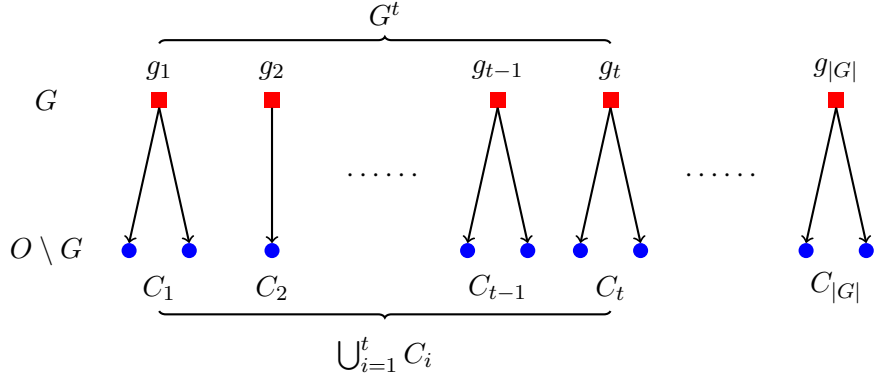
\begin{figure}[!t]
    \centering

\begin{tikzpicture}
\usetikzlibrary{patterns,snakes}
\newcommand*{\BlockWidth}{0.1}%
\pgfmathsetmacro{\Radius}{\BlockWidth}%
\pgfmathsetmacro{\Xoff}{4*\BlockWidth}%
\pgfmathsetmacro{\Yoff}{20*\BlockWidth}%

\node at (-6, 0) {$G$};
\node at (-6, -\Yoff) {$O \setminus G$};

\foreach \x in {-4.5, 0, 1.5, 4.5}
\foreach \y in {0}
{
	\path [fill=red] (\x-\BlockWidth,\y-\BlockWidth) rectangle (\x+\BlockWidth, \y+\BlockWidth);
 	\path [fill=blue] (\x - \Xoff, \y - \Yoff) circle (\Radius);
 	\path [fill=blue] (\x + \Xoff, \y - \Yoff) circle (\Radius);
 	\draw [->, thick] (\x, \y-\BlockWidth) to (\x - \Xoff, \y - \Yoff + \Radius);
 	\draw [->, thick] (\x, \y-\BlockWidth) to (\x + \Xoff, \y - \Yoff + \Radius);
 }
 \node (g1) at (-4.5, 0.4) {$g_1$};   \node (C1) at (-4.5, -\Yoff-0.5) {$C_1$};
 \node (gt1) at (0, 0.4) {$g_{t-1}$};   \node (Ct1) at (0, -\Yoff-0.5) {$C_{t-1}$};
 \node (gt) at (1.5, 0.4) {$g_t$};      \node (Ct) at (1.5, -\Yoff-0.5) {$C_t$};
 \node at (4.5, 0.4) {$g_{|G|}$}; \node at (4.5, -\Yoff-0.5) {$C_{|G|}$};
 
  \node at (-1.5, -\Yoff/2) {$\cdots\cdots$};
  \node at (3, -\Yoff/2) {$\cdots\cdots$};
 
 \foreach \x in {-3}
\foreach \y in {0}
{
	\path [fill=red] (\x-\BlockWidth,\y-\BlockWidth) rectangle (\x+\BlockWidth, \y+\BlockWidth);
 	\path [fill=blue] (\x, \y - \Yoff) circle (\Radius);
 	\draw [->, thick] (\x, \y-\BlockWidth) to (\x, \y - \Yoff + \Radius);
 }
  \node at (-3, 0.4) {$g_2$};   \node at (-3, -\Yoff-0.5) {$C_2$};
  
 \draw [
    thick,
    decoration={
        brace,
        mirror,
        raise=0.6cm
    },
    decorate
] (C1.north) -- (Ct.north) 
node [pos=0.5,anchor=north,yshift=-0.8cm] {$\bigcup_{i=1}^t C_i$}; 

 \draw [
    thick,
    decoration={
        brace,
        raise=0.1cm
    },
    decorate
] (g1.north) -- (gt.north) 
node [pos=0.5,anchor=north,yshift=0.8cm] {$G^t$}; 
 
\end{tikzpicture}

    \caption{Notation for analyzing Algorithm~\ref{alg:greedyFixedDensity}.
    The elements in the greedy solution $\Greedy$ are arranged according to the order of being selected in Step 3 in Algorithm~\ref{alg:greedyFixedDensity}.
    The elements in the optimal solution $O$ but not in the greedy solution $G$ are partitioned into groups $C_t (1 \leqslant t \leqslant |G|)$, where $C_t$ are those elements in $O\setminus G$ that are still feasible before selecting $g_t$ but are infeasible after selecting $g_t$.
    } \label{fig:greedyPartition}
\end{figure}

First, by the property of the intersection of matroids,
the size of the prefix $\bigcup_{i=1}^t C_t$ is bounded by $Pt$.
The property is that for any $Q \subseteq \Ground$, the sizes of any two maximal independent subsets $T_1$ and $T_2$ of $Q$
can only differ by a multiplicative factor at most $P$.
To see this, note that for any element $ z \in T_1 \setminus T_2$,
$\cbr{z} \cup T_2$ violates at least one of the matroid constraints since $T_2$ is maximal.
Let $V_i (1 \leqslant i \leqslant P)$ denote all elements in $T_1 \setminus T_2$ that violates the $i$-th matroid,
and then partition $T_1 \cap T_2$ arbitrarily among these $V_i$'s so that they cover $T_1$.
Note that the size of each $V_i$ must be at most that of $T_2$,
since otherwise by the Exchange axiom, there would exist $z \in V_i \setminus T_2$ that
can be added to $T_2$ without violating the $i$-th matroid, which is contradictory to the construction.
Therefore, the size of $T_1$ is at most $P$ times that of $T_2$.

To apply this property, let $Q$ be the union of $G^{t}$ and $\bigcup_{i=1}^t C_t$.
On one hand, $G^{t}$ is a maximal independent subset of $Q$, since no element in $\bigcup_{i=1}^t C_t$ can be added to $G^t$ without violating the matroid constraints.
On the other hand, $\bigcup_{i=1}^t C_t$ is an independent subset of $Q$, since it is part of the optimal solution.
Therefore, $\bigcup_{i=1}^t C_t$ has size at most $P$ times $|G^t|$, which is $Pt$.
Note that the properties of matroids are crucial for this analysis,
which justifies our formulation using matroids.
In summary, we have

\begin{claim}\label{cla:size}
$\sum_{i=1}^t |C_i| \leqslant P t$, for $t =1, \dots, |G|$.
\end{claim}

Second, we compare the marginal gain of each element in $C_t$ to that of $g_t$.
Suppose $g_t$ is selected at the threshold $\tau_t>0$.
Then any $j \in C_t$ has marginal gain bounded by $(1+\delta)\tau_t + 2\epsilon$,
since otherwise $j$ would have been selected at a larger threshold before $\tau_t$ by the greedy criterion.
Now suppose $g_t$ is selected at the threshold $w_{L+1}=0$.
Then the marginal gain of any $j \in C_t$ is approximately bounded by $w_{L+1} \leqslant \frac{\delta}{\nGround} d$.
Since the greedy algorithm must pick $g_1$ with $\widehat f(g_1) = d$,  $d \leqslant f(g_1)+\epsilon$,
and the gain of $j$ is bounded by $\frac{\delta}{\nGround} f(G) + O(\epsilon)$.
All together:
\begin{claim}\label{cla:gain}
Suppose $g_t$ is selected at the threshold $\tau_t$. Then $f(j|G^{t-1}) \leqslant (1+\delta) \tau_t + 4\epsilon + \frac{\delta}{\nGround} f(G), \forall j \in C_t$.
\end{claim}
Note that the evaluation of the marginal gain of $g_t$ should be at least $\tau_t$,
so this claims essentially says that the marginal gain of $j$ is approximately bounded by that of $g_t$.

As there are not many elements in $C_t$ (Claim~\ref{cla:size}) and the marginal gain of each element in it
is not much larger than that of $g_t$ (Claim~\ref{cla:gain}),
the marginal gain of $O \setminus G = \bigcup_{i=1}^{|G|} C_t$ is not much larger than that of $G$,
which is just $f(G)$.

\begin{claim}\label{cla:com}
The marginal gain of $O\setminus G$ satisfies
$$\sum_{j \in O\setminus G} f(j|\Greedy) \leqslant [(1+\delta) P + \delta] f(G)  + (6+2\delta) \epsilon P |G|. $$
\end{claim}

Since by submodulairty, $f(O) \leqslant f(O\cup G) \leqslant f(G) + \sum_{j \in O\setminus G} f(j|\Greedy)$,
Claim~\ref{cla:com} essentially shows $f(G)$ is close to $f(O)$ up to a multiplicative factor roughly $(1+P)$  and aditive factor $O(\epsilon P |G|)$.
Since $f(G) > |G|$, it leads to roughly $1/3$-approximation for our influence maximization problem by setting $\epsilon = \delta/16$ when evaluating $\widehat f$ with \continmax. Combining the above analysis and the running time of \continmax~\citep{DuSonZhaGom13}, we have our final guarantee in Theorem~\ref{thm:infMax_uni}.

\subsection{General Case}\label{sec:nonuni}

Here we consider the more general and more challenging case when the users may have different costs. Recall that this case corresponds to Problem~\ref{pro:infMax} with $P=1$ matroid constraints and $k=|\Item|$ group-knapsack constraints. We show that in Algorithm~\ref{alg:densityEnu}, there is a step which outputs a solution $S_\rho$ that is a good approximation.

\noindent{\bf Intuition.} The key idea behind Algorithm~\ref{alg:densityEnu} and Algorithm~\ref{alg:greedyFixedDensity} is simple:
spend the budgets efficiently and spend them as much as possible.
To spend them efficiently, we only select those elements whose density ratio between the marginal gain and the cost is above the threshold $\rho$.
That is, we assign product $i$ to user $j$ only if the assignment leads to large marginal gain without paying too much.
To spend the budgets as much as possible, we stop assigning product $i$ only if its budget is almost exhausted or
no more assignment is possible without violating the matroid constraints.
Here we make use of the special structure of the knapsack constraints on the budgets:
each constraint is only related to the assignment of the corresponding product and its budget,
so that when the budget of one product is exhausted, it does not affect the assignment of the other products.
In the language of submodular optimization, the knapsack constraints are on a partition $\Ground_{i*}$ of the ground set
and the objective function is a sum of submodular functions over the partition.
For general knapsack constraints without such structure, it may not be possible to continue selecting elements as in our case.

However, there seems to be an hidden contradiction between spending the budgets efficiently and spending them as much as possible:
on one hand, efficiency means the density ratio should be large, so the threshold $\rho$ should be large;
on the other hand, if $\rho$ is large, there are just a few elements that can be considered, then the budget might not be exhausted.
After all, if we set $\rho$ to be even larger than the maximum possible, then no element is considered and no gain is achieved.
In the other extreme, if we set $\rho=0$ and consider all the elements, then a few elements with large costs might be selected, exhausting all the budgets and leading to a poor solution.

It turns out that there exists a suitable threshold $\rho$ achieving a good balance between the two and leads to good approximation.
The threshold is sufficiently small, so that the optimal elements we abandon (\ie, those with low density ratio) have a total gain at most a fraction of the optimum.
It is also sufficiently large, so that the elements selected are of high quality (\ie, of high density ratio), and we must have sufficient gain if the budgets of some items are exhausted. Formally, for our influence maximization problem,

\begin{theorem}\label{thm:infMax}
In Algorithm~\ref{alg:densityEnu}, there exists a $\rho$ such that
$$
    f(S_\rho) \geqslant  \frac{\max\cbr{k_a, 1} }{(2|\Item|+2) (1+3\delta)} f(O)
$$
where $k_a$ is the number of active knapsack constraints.
The expected running time is $\widetilde\Ocal\left(\frac{|\Edge_{i^*}|+|\Node|}{\delta^2}  + \frac{|\Item||\Node|}{\delta^4} \right).$
\end{theorem}

The approximation factor improves over the best known guarantee $\frac{1}{P+2k + 1} = \frac{1}{2|\Item| + 2}$ for effciently maximizing submodular functions over $P$ matroids and $k$ general knapsack constraints. As in the uniform cost case, the parameter $\delta$ introduces a tradeoff between the approximation and the runtime. Since the runntime has a linear dependence on the network size, the algorithm easily scales to large networks.

\noindent{\bf Analysis.} The analysis follows the intuition. Pick $\rho = \frac{2 f(O)}{P+2k+1}$ where $O$ is the optimal solution. Define 
\begin{align*}
    O_- & := \cbr{z \in O\setminus S_\rho: f(z| S_\rho) < c(z) \rho + 2\epsilon }, \\
    O_+ & := \cbr{z \in O\setminus S_\rho: z\not\in O_-}.
\end{align*}
By submodularity, $O_-$ is a superset of the elements in the optimal solution that we abandon due to the density threshold.
By construction, its marginal gain is small:
$$
    f(O_- | S_\rho) \leqslant \rho c(O_-) + \Ocal(\epsilon |S_\rho|) \leqslant k\rho + + \Ocal(\epsilon |S_\rho|)
$$
where the small additive term $\Ocal(\epsilon |S_\rho|)$ is due to inexact function evaluations.

First, if no knapsack constraints are active, then the algorithm runs as if there were no knapsack constraints
(but only on elements with density ratio above $\rho$).
So we can apply the argument for the case with only matroid constraints (see the analysis up to Claim~\ref{cla:com} in Section~\ref{sec:uni});
but we apply it on $O_+$ instead of on $O\setminus S_\rho$.
Similar to Claim~\ref{cla:com}, we have
$$
    f(O_+| S_\rho) \leqslant [(1+\delta)P+\delta ] f(S_\rho) + \Ocal(\epsilon P |S_\rho|)
$$
where the small additive term $\Ocal(\epsilon P |S_\rho|)$ is due to inexact function evaluations.
By the fact that $f(O) \leqslant f(S_\rho) + f(O_-|S_\rho) + f(O_+ |S_\rho)$,
we know that $S_\rho$ is roughly a $\frac{1}{P+ 2k + 1}$-approximation.

Second, suppose $k_a >0$ knapsack constraints are active.
Suppose the algorithm discovers that the budget of product $i$ is exhausted
when trying to add element $z$, and the elements selected for product $i$ at that time is $G_i$.
Since $c(G_i \cup \cbr{z}) > 1$ and each of these elements has density above $\rho$,
the gain of $G_i \cup \cbr{z}$ is above $\rho$.
However, only $G_i$ is included in our final solution, so we need to show that the marginal gain of $z$ is not large compared to that of $G_i$.
In fact, the algorithm greedily selects elements with marginal gain above a decreasing threshold $w_t$.
Since $z$ is the last element selected and $G_i$ is nonempty (otherwise adding $z$ will not exhaust the budget),
the marginal gain of $z$ must be bounded by roughly that of $G_i$.
In summary, the gain of $G_i$ is at least roughly $\frac{1}{2}\rho$.
This holds for all active knapsack constraints,
so the solution has value at least $\frac{k_a }{2}\rho$, which is an $\frac{k_a }{P+2k+1}$-approximation.

Combining the two cases, and setting $k=|\Item|$ and $P=1$ as in our problem, we have our final guarantee in Theorem~\ref{thm:infMax}.

\section{Experiments} \label{sec:exp}

We systematically evaluate the performance and scalability of our algorithm, denoted by \budgetmax, on both the synthetic datasets mimicking the structural properties of real-world networks and the real Memetracker datasets~\citep{LesBacKle09} crawled from massive media-sites. We compare \budgetmax to its counterpart based on the learned classic discrete-time diffusion model, the particularly designed degree-based heuristics,  as well as the random baseline to show that \budgetmax achieves significant performance gains in both cases.
\subsection{Synthetic Data}
\noindent{\bf Synthetic Diffusion Network Generation.} We assume that products have different diffusion network structures. In particular, we allow each product to spread over one of the following three different types of Kronecker networks\citep{LesChaKleFaletal10}: (\rmnum{1}) core-periphery networks (parameter matrix: [0.9 0.5; 0.5 0.3]) mimicking the diffusion traces of information in real world networks~\citep{GomLesKra10},
(\rmnum{2}) the classic random networks ([0.5 0.5; 0.5 0.5]) used in physics and graph theory~\citep{EasKle10}
as well as (\rmnum{3}) hierarchical networks ([0.9 0.1; 0.1 0.9])~\citep{GomBalSch11}. Once the network structure is generated, we assign a general Weibull
distribution~\citep{JerLaw02} with randomly chosen parameters from 1 to 10 in order to have heterogeneous temporal dynamics.
In our experiments we have 64 products, each of which diffuses over one of the above three different types of networks with 1,048,576 nodes. Then, we further randomly select a subset $\Vcal_S\subseteq\Vcal$ of 512 nodes as our candidate target users who will receive the given 64 products. The potential influence of an allocation will be evaluated over the underlying one-million-node networks.

\subsubsection{Influence Maximization with Uniform Costs}

\noindent{\bf Competitors}.We compare \budgetmax with nodes' degree-based heuristics of the diffusion network which are usually applied in social network analysis, where the degree is treated as a natural measure of influence. Large-degree nodes, such as users with millions of followers in Twitter, are often the targeted users who will receive a considerable payment if he (she) agrees to post the adoption of some products (or ads) from merchants. As a consequence,  we first sort the list of all pairs of product $i$ and node $j\in\Vcal_S$ in the descending order of node-$j$'s degree in the diffusion network of product $i$. Then, starting from the beginning of the list, we add each pair one by one. When the addition of the current pair to the existing solution violates the predefined matroid constraints, we simply throw it and continue to search the next pair until we reach the end of the list. Therefore, we greedily assign products to the nodes with large degree, and we refer to this heuristic as \gdegree.  Finally, we consider the baseline method that assigns the products to the target nodes randomly.  Due to the large size of the underlying diffusion networks, we do not apply other more expensive node centrality measures such as the clustering coefficient and betweenness.

\begin{figure*}[t]
 \centering
 \renewcommand{\tabcolsep}{0pt}
 \begin{tabular}{ccccc}
\includegraphics[width=0.2\textwidth]{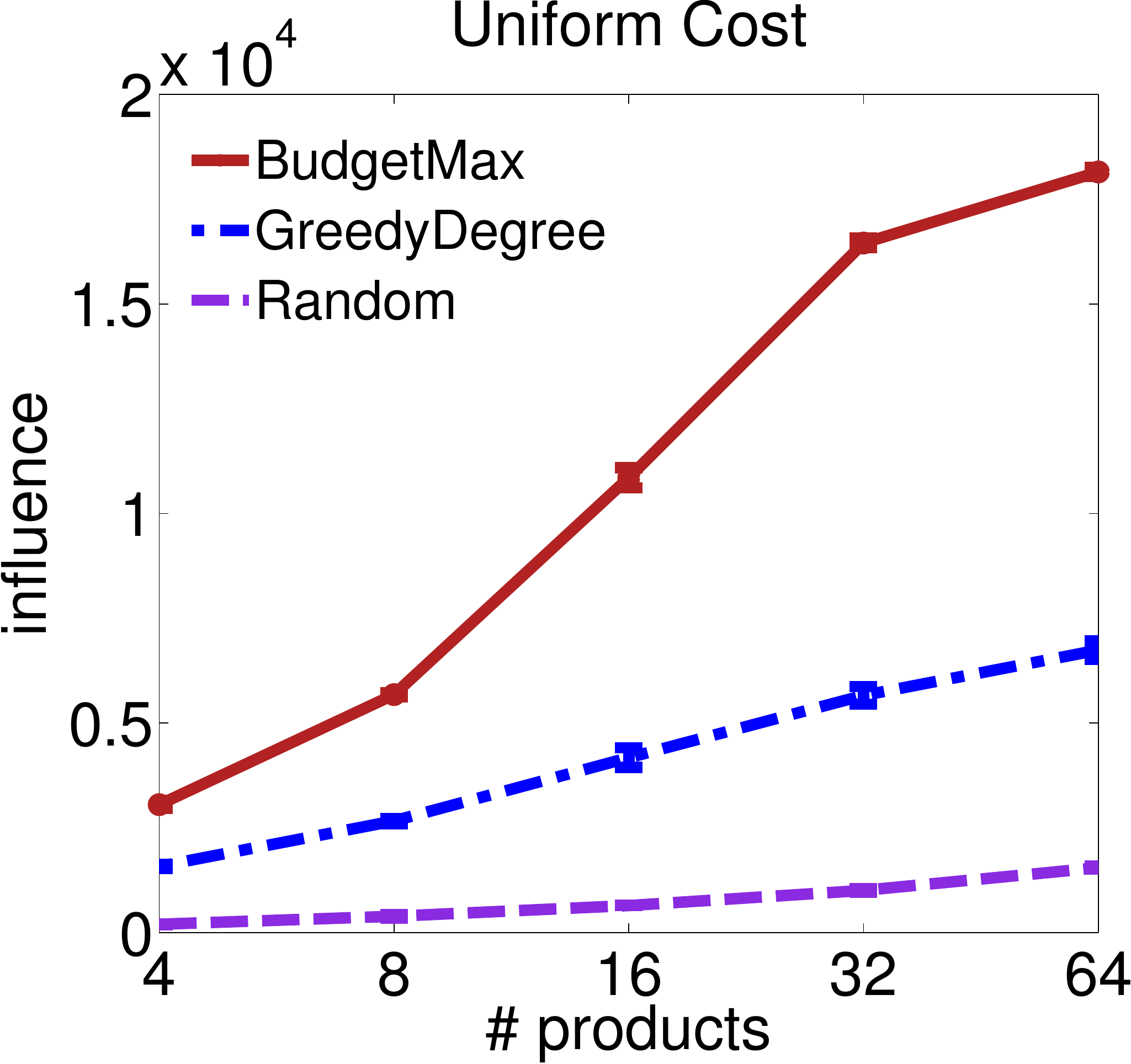} &
\includegraphics[width=0.2\textwidth]{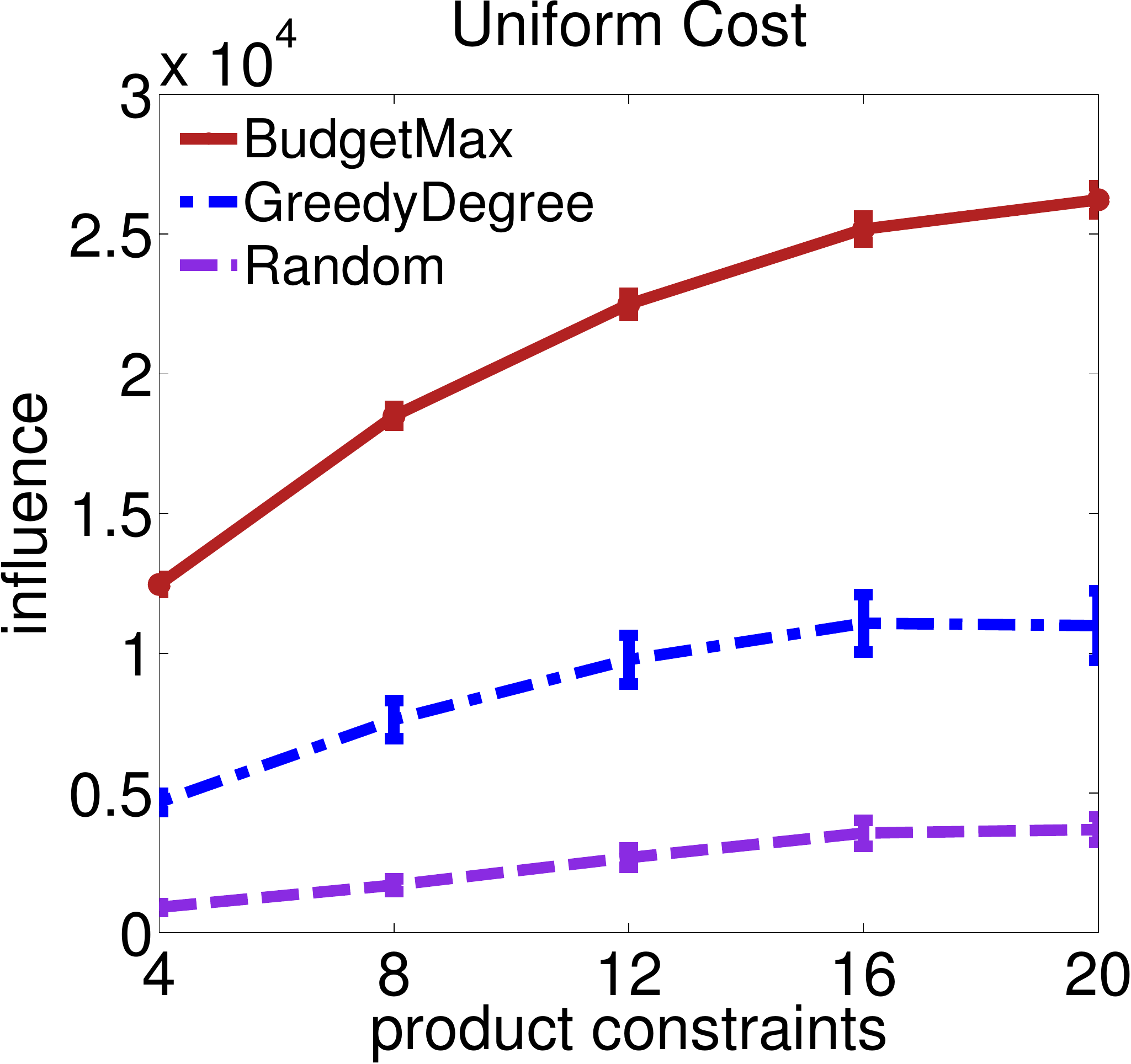} &
\includegraphics[width=0.2\textwidth]{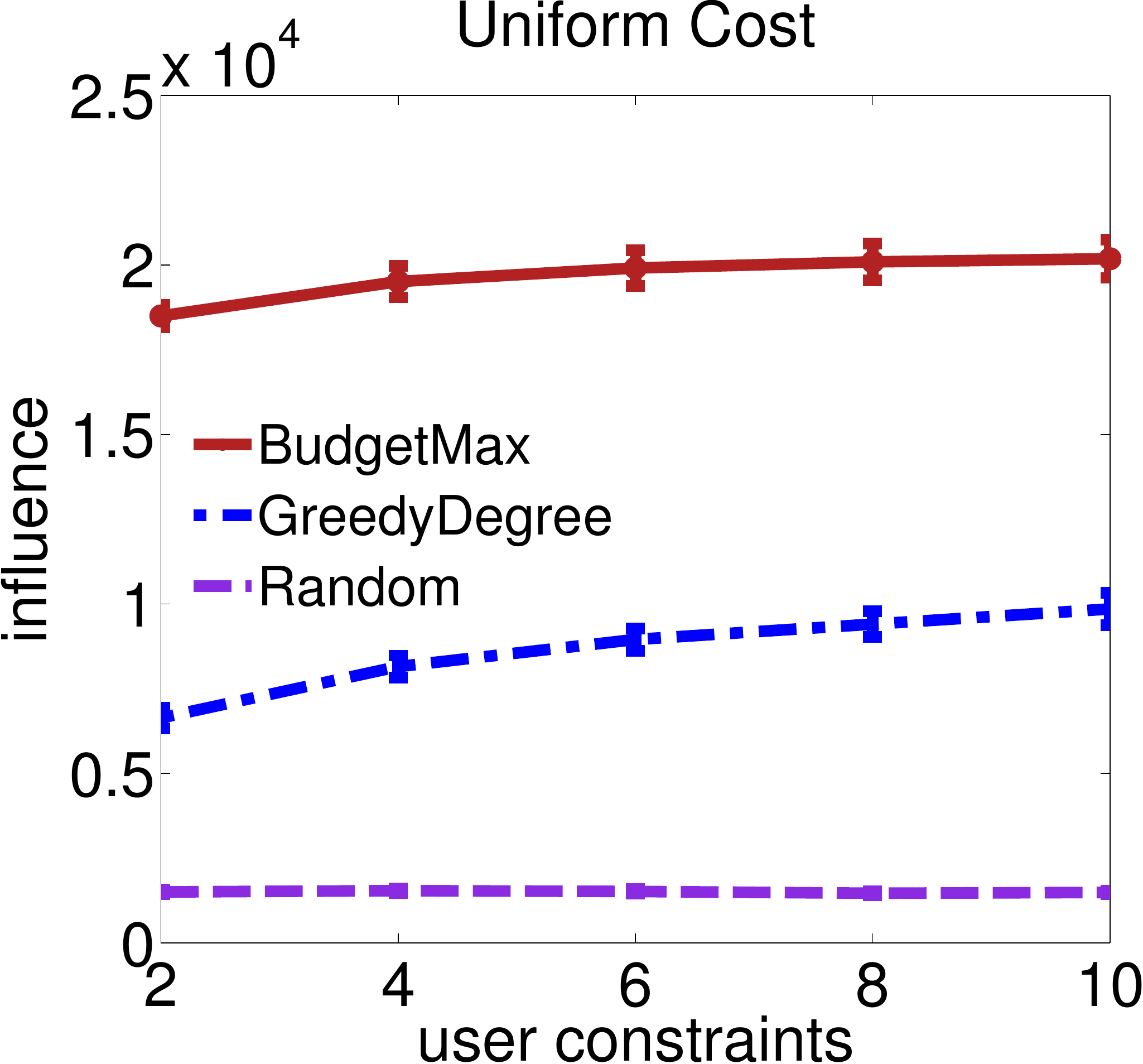} &
\includegraphics[width=0.2\textwidth]{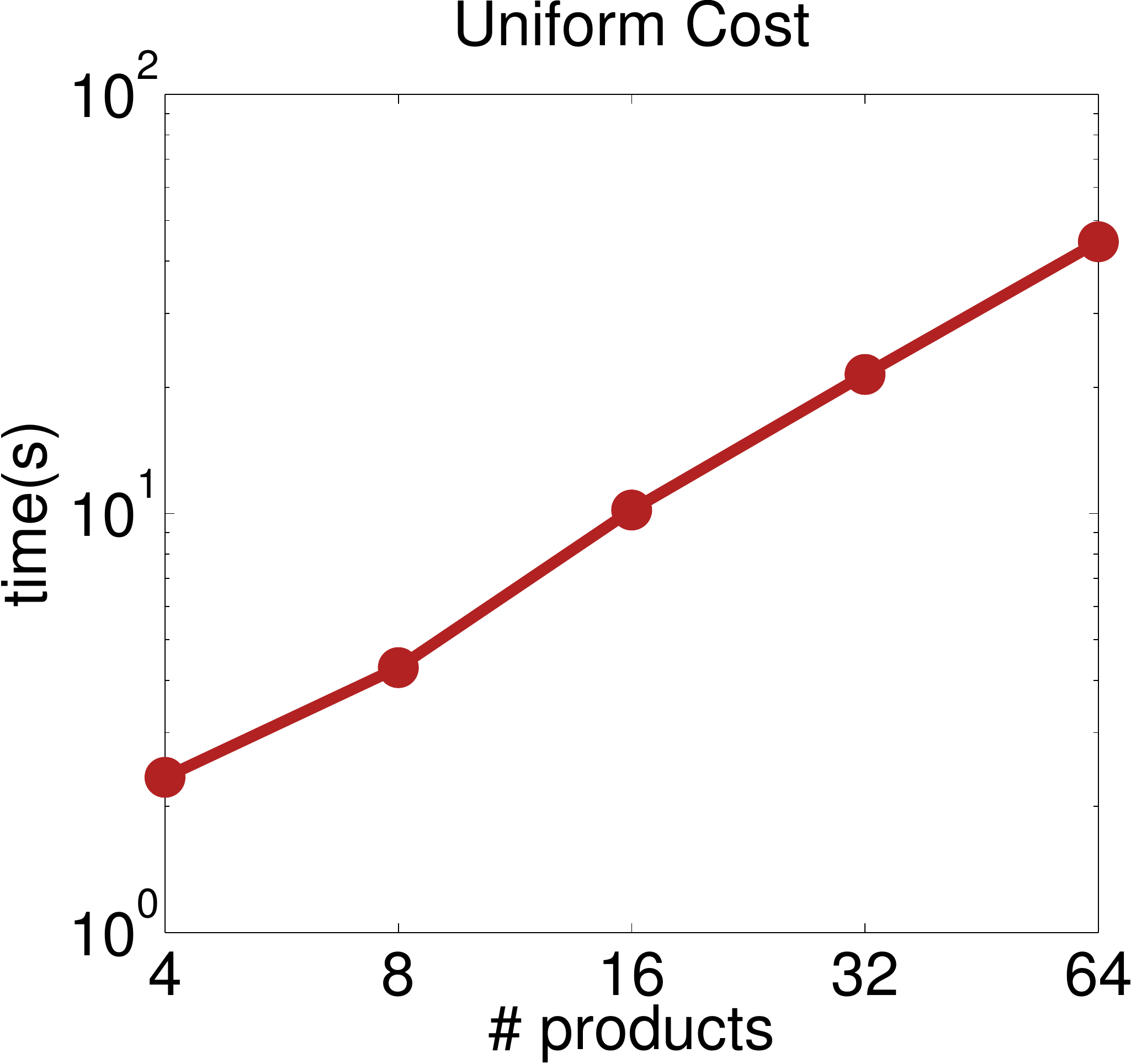} &
\includegraphics[width=0.2\textwidth]{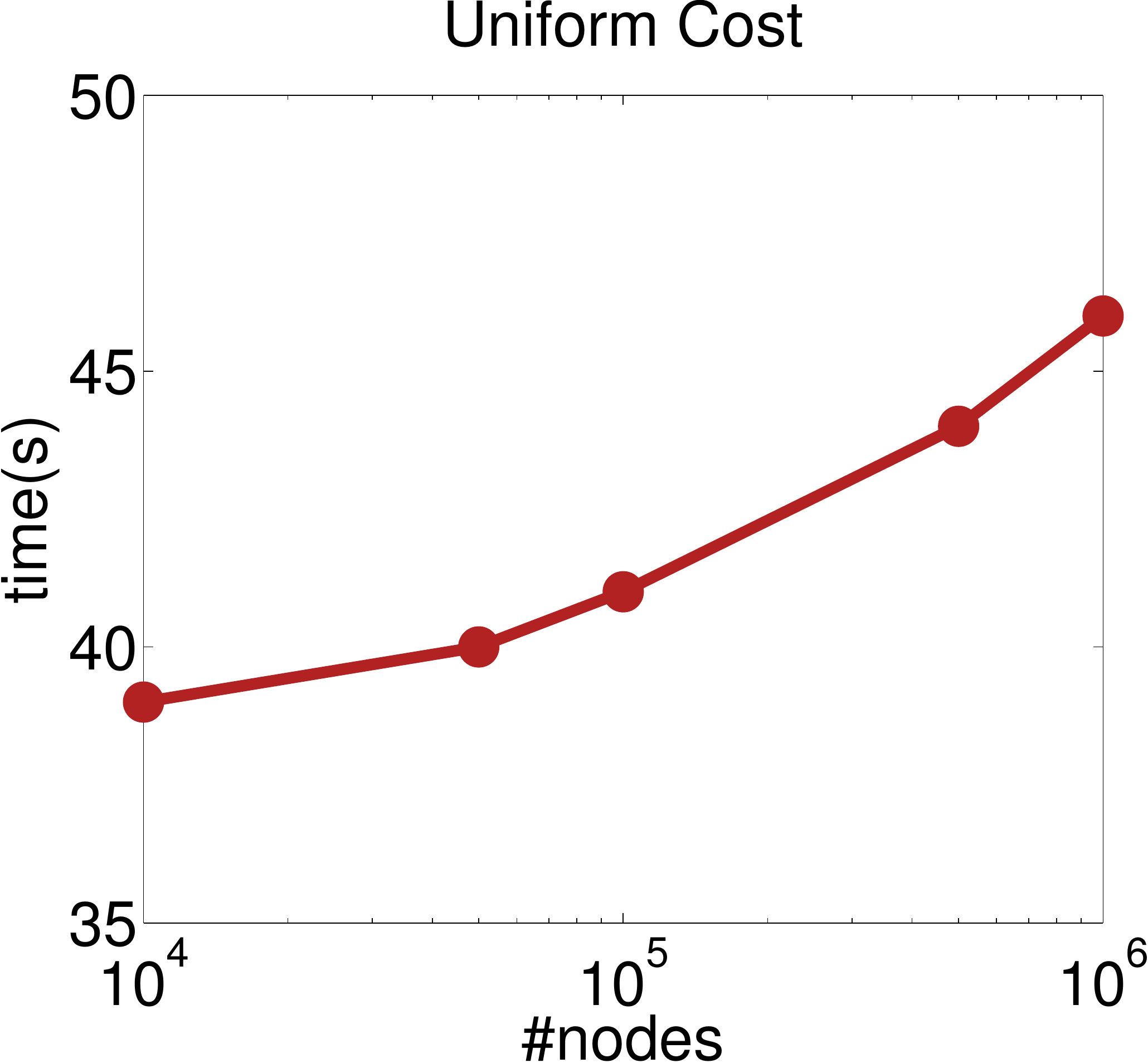} \\
(a) By products & (b) By product constraints & (c) By user constraints & (d) Speed by products & (e) Speed by nodes
\end{tabular}
 \caption{\label{inf-syn} Over the 64 product-specific diffusion networks, each of which has 1,048,576 nodes,  the estimated influence (a) for increasing the number of products by fixing the product-constraint at 8 and user-constraint at 2; (b) for increasing product-constraint by user-constraint at 2; and (c) for increasing user-constraint by fixing product-constraint at 8. Fixing product-constraint at 8 and user-constraint at 2, runtime (d)  for allocating increasing number of products  and (e) for allocating 64 products to 512 users on networks of varying size. For all experiments, we have $T=5$ time window.}
\end{figure*}

\noindent{\bf Influence Maximization.}. On each of the 64 product-specific diffusion networks, we generate a set of 2,048 samples to estimate the influence of each node according to~\citep{DuSonZhaGom13}. We repeat our experiments for 10 times and report the average performance in Figure~\ref{inf-syn} in which the adaptive threshold $\delta$ is set to 0.01. First, Figure~\ref{inf-syn}(a) compares the achieved influence by increasing the number of available products, each of which has constraint 8. As the number of products increases, on the one hand, more and more nodes become assigned, so the total influence will increase. Yet, on the other hand, the competitions for a few existing valuable nodes from which information diffuses faster also increases. For \gdegree, because high degree nodes may have many overlapping children and highly clustered, the marginal gain by targeting only these nodes could be small. In contrast, by taking both the network structure and the diffusion dynamics of the edges into consideration, \budgetmax is able to find allocations that could reach as many nodes as possible as time
unfolds. In Figure~\ref{inf-syn}(b), we fix the set of 64 products while increasing the number of budget per product. Again, as the competitions increase, the performance of \gdegree tends to converge, while the advantage of \budgetmax becomes more dramatic. We investigate the effect of increasing the user constraint while fixing all the other parameters. As Figure~\ref{inf-syn}(c) shows, the influence increases slowly for that fixed budget prevents additional new nodes to be assigned. This meets our intuition for that only making a fixed number of people watching more ads per day can hardly boost the popularity of the product. Moreover, even though the same node can be assigned to more products, because of the different diffusion structures, it cannot be the perfect source from which all products can efficiently spread.

\noindent{\bf Scalability}. We further investigate the performance of \budgetmax in terms of runtime when using~\continmax\citep{DuSonZhaGom13} as subroutine to estimate the influence. We can precompute the data structures and store the samples needed to estimate the influence function in advance. Therefore, we focus only on the runtime for the constrained influence maximization algorithm.
\begin{figure}[t]
 \centering
 \renewcommand{\tabcolsep}{5pt}
 \begin{tabular}{ccc}
\includegraphics[width=0.3\columnwidth]{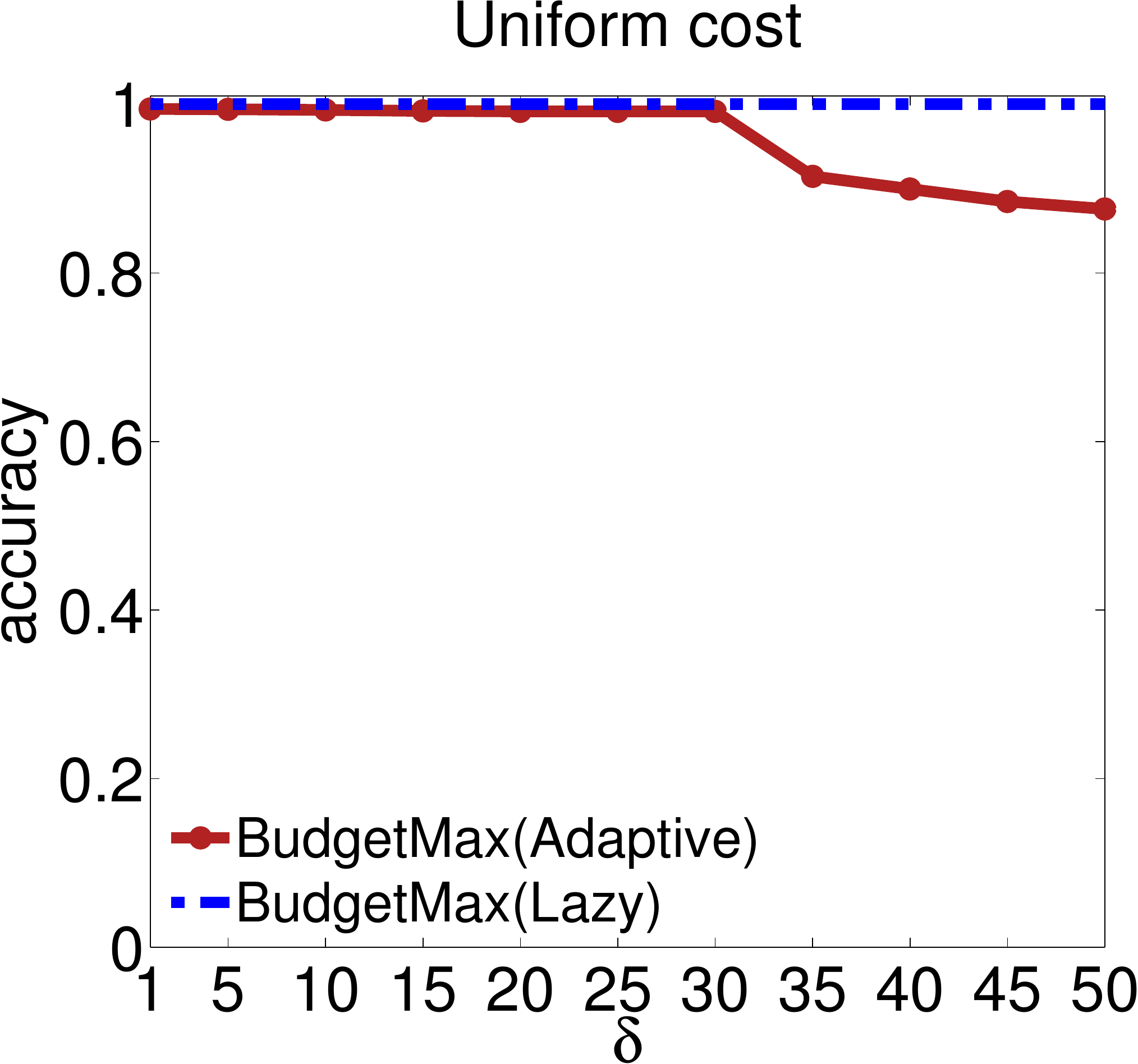}
& \hspace{1cm} 
&\includegraphics[width=0.3\columnwidth]{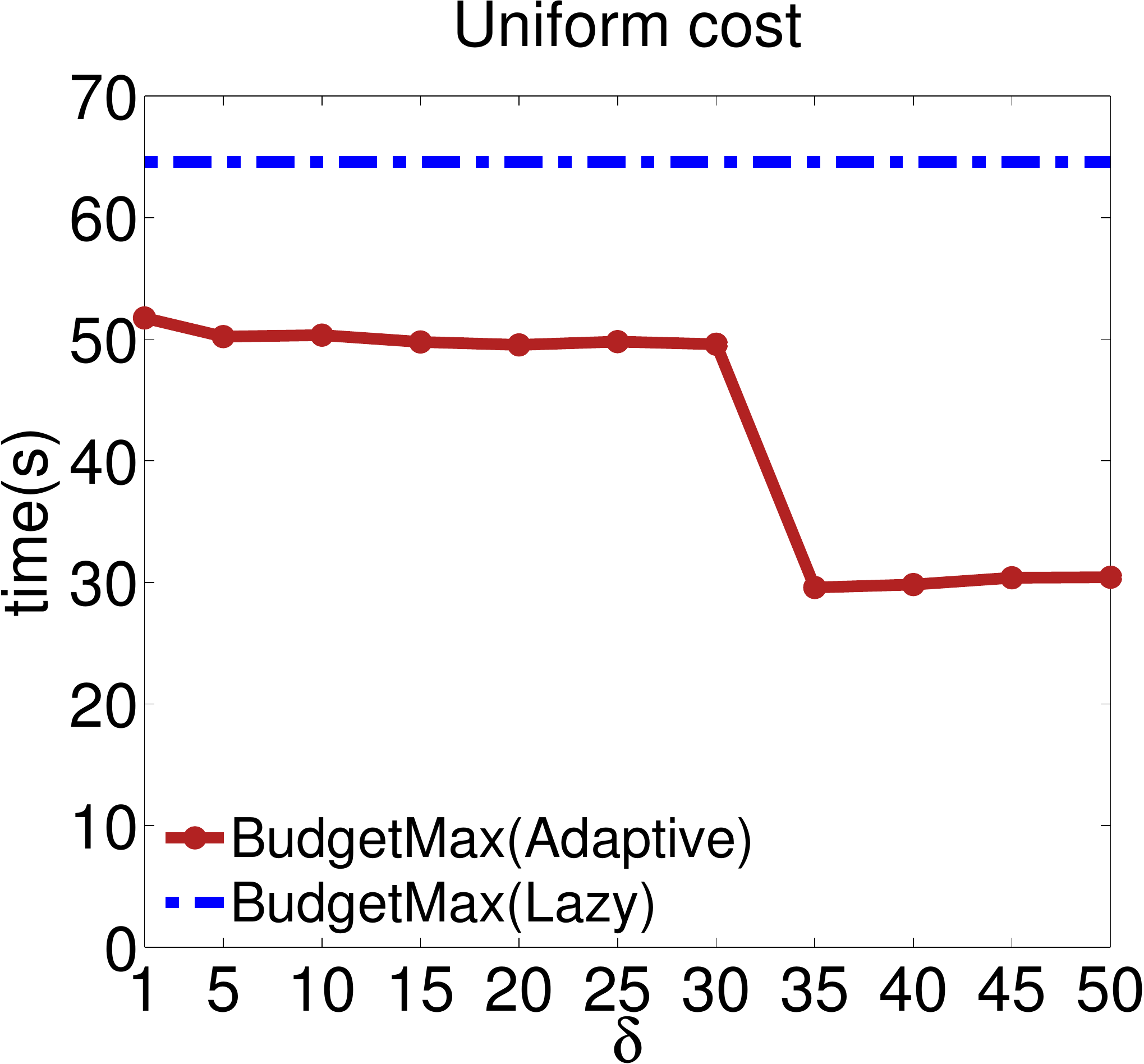} \\
(a) $\delta$ vs. accuracy & & (b) $\delta$ vs. time
\end{tabular}
 \caption{\label{speed-syn}
The relative accuracy and the run-time for different threshold parameter $\delta$.}
\end{figure}
\budgetmax runs on 64 cores of 2.4Ghz by using OpenMP to accelerate the first round of the optimization.
We report the allocation time for increasing number of products in Figure~\ref{inf-syn}(d), which clearly shows a linear time complexity with respect to the size of the ground set. Figure~\ref{inf-syn}(e) evaluates the runtime of allocation by varying the size of the network from 16,384 to 1,048,576 nodes.
We can see that \budgetmax~can scale up to millions of nodes.

\noindent{\bf Effects of Adaptive Thresholding}. In Figure~\ref{speed-syn}(a), we compare our adaptive thresholding algorithm to the lazy evaluation method. We plot the achieved influence value by different threshold $\delta$ relative to that achieved by the lazy evaluation method. Since the lazy evaluation method does not depend on the parameter, it is always 1 shown by the blue line. We can see that as $\delta$ increases, the accuracy will decrease. However, the performance is robust to $\delta$ in the sense that we can still keep 90-percent relative accuracy even if we use large $\delta$. Finally, in Figure~\ref{speed-syn}(b), we show that as $\delta$ increases, the runtime can be significantly reduced. Thus, Figure~\ref{speed-syn} verifies the intuition that  $\delta$ is able to trade off the solution quality of the allocation with the runtime. The larger $\delta$ becomes, the shorter the runtime will be, at the cost of  reduced allocation quality.

\subsubsection{Influence Maximization with Non-Uniform Costs}

\noindent{\bf User-cost and product-budget generation}. Our designing of user-cost mimics the real scenario where advertisers pay much more money to celebrities with millions of social network followers by letting $c_i\propto d_i^{-n}$ where $c_i$ is the cost, $d_i$ is the degree, and $n\geqslant 1$ controls the increasing speed of cost w.r.t degree. In our experiments, we use $n=3$ and normalize $c_i$ to be within $[0,1]$. Then, the product-budget consists of a base value from 1 to 10 with a random adjustment uniformly chosen from 0 to 1.

\noindent{\bf Competitors}. Because users now have non-uniform costs, \gdegree should take both degree and the corresponding cost into consideration. Hence, we sort the list of all pairs of product $i$ and node $j\in\Vcal_S$ in the descending order of the degree-cost ratio $d_j/c_j$ in the corresponding diffusion networks to select the most cost-effective pairs. In addition, if we allow the target users to be partitioned into distinct groups (or communities), we can also allow the group-specific allocation. In particular, we may pick the most cost-effective pairs within each group locally instead, which is referred to as \gdegree(local).

\noindent{\bf Influence Maximization}. We denote our experimental settings by the tuple (\#products, product-budget, user-constrain, T). In Figure~\ref{inf-budget-syn}(a-d), we investigate the relation between the estimated influence and one of the above four factors while fixing the others constant each time. In all cases, \budgetmax significantly outperforms the other methods, and the achieved influence increases monotonically. 

\noindent{\bf Group limits}. In Figure~\ref{inf-budget-syn}(e), we study the effect of the Laminar matroid combined with group knapsack constraints, which is the most general type of constraint we handle in this paper. The selected target users are further partitioned into $K$ groups randomly, each of which has $Q_i,i = 1\dotso K$ limit which constrains the maximum allocations allowed in each group. In practical scenarios, each group might correspond to a geographical community or organization. In our experiment, we divide the users into 8 equal-size groups and set $Q_i = 16,i = 1\dotso K$ to indicate that we want a balanced allocation in each group. Figure~\ref{inf-budget-syn}(e) shows the estimated influence with respect to the user-constraint. In contrast to Figure~\ref{inf-budget-syn}(b), as we increase the user-constraint by giving more slots to each user, the total estimated influence keeps almost constant. This is because although the total number of available slots in each group increases, the group limit does not change. As a consequence, we still cannot make more allocations to increase the total influence.

\begin{figure*}[t]
 \centering
 \renewcommand{\tabcolsep}{0pt}
 \begin{tabular}{ccccc}
\includegraphics[width=0.2\textwidth]{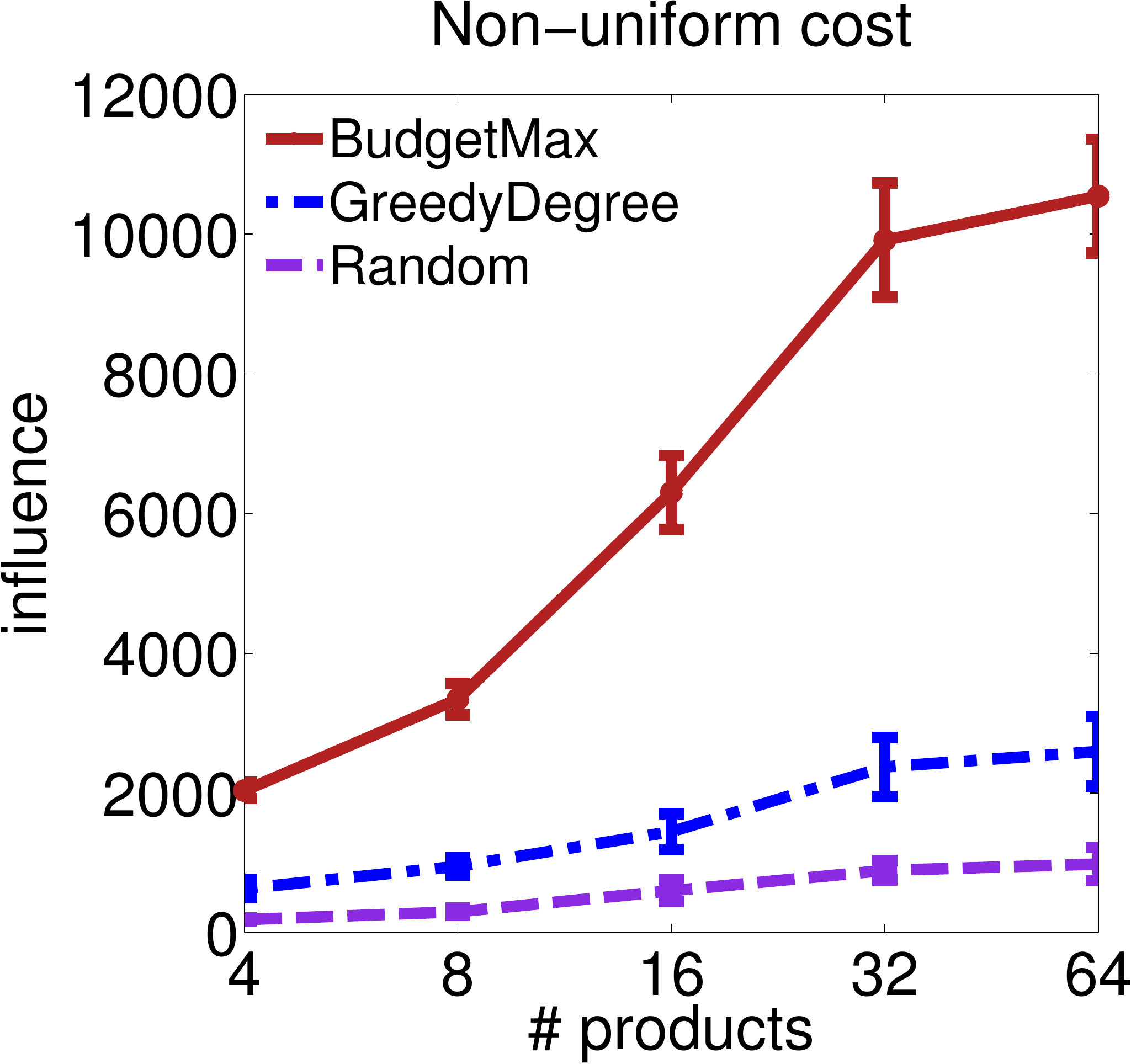} &
\includegraphics[width=0.2\textwidth]{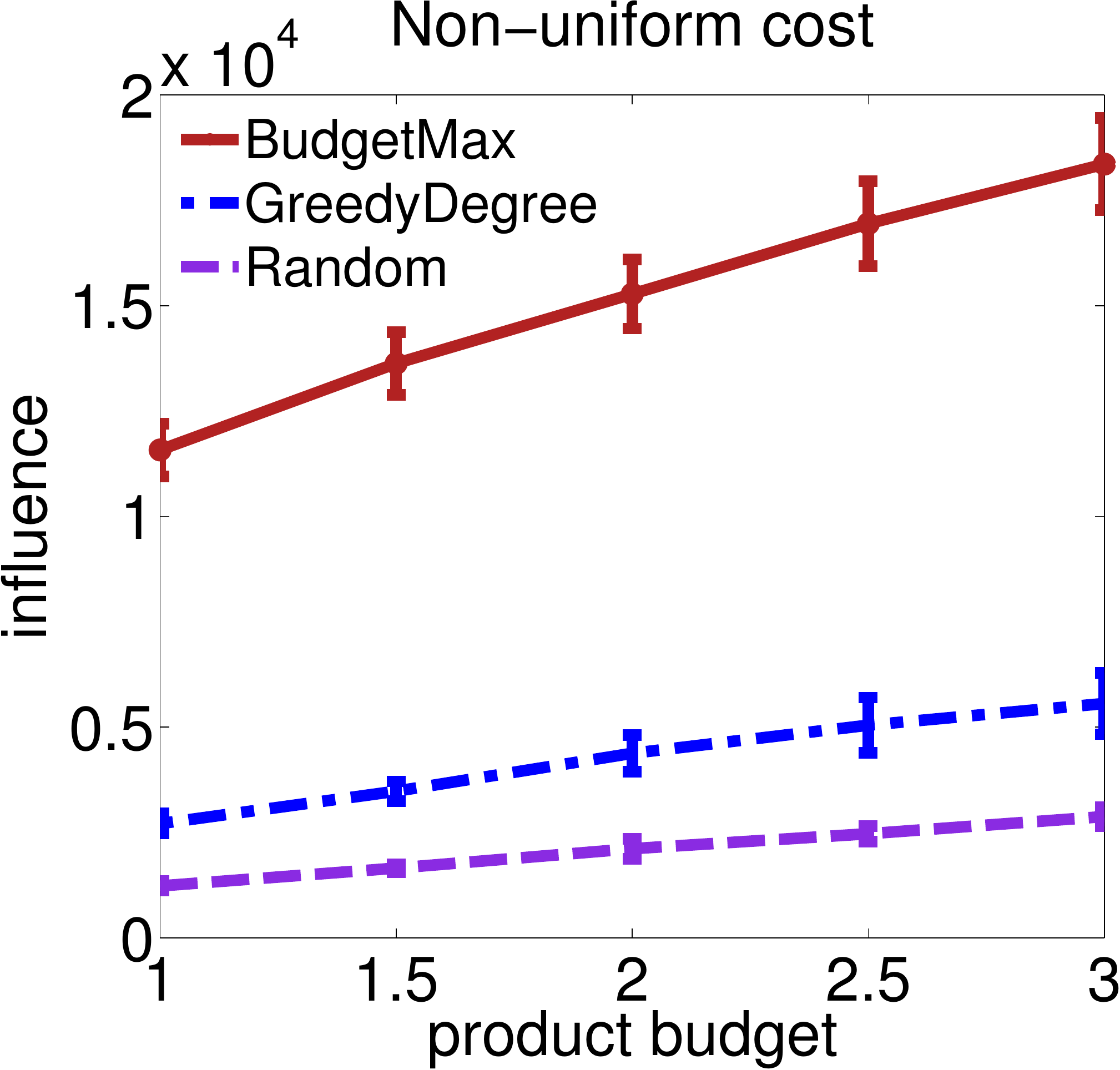} &
\includegraphics[width=0.2\textwidth]{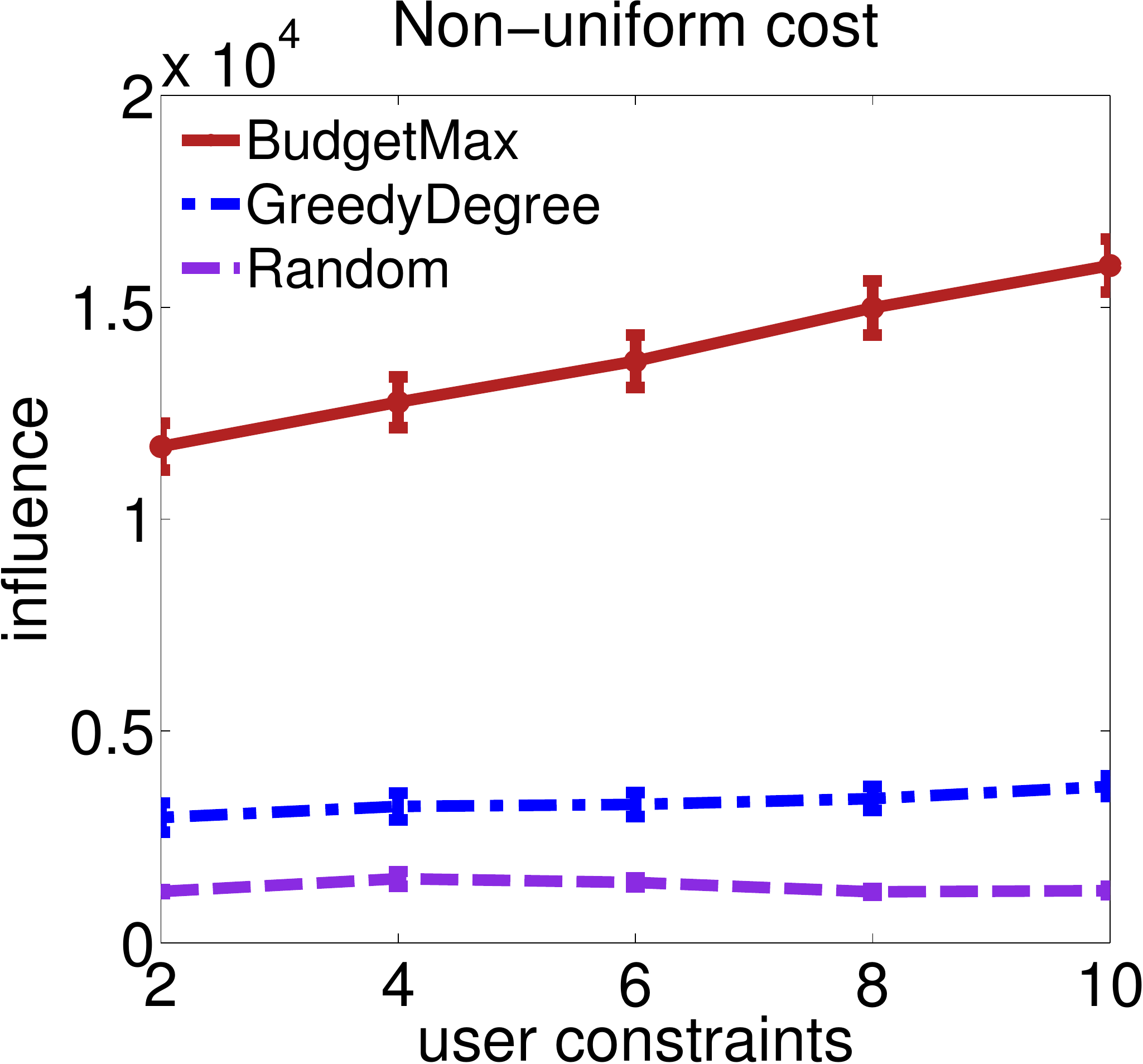} &
\includegraphics[width=0.2\textwidth]{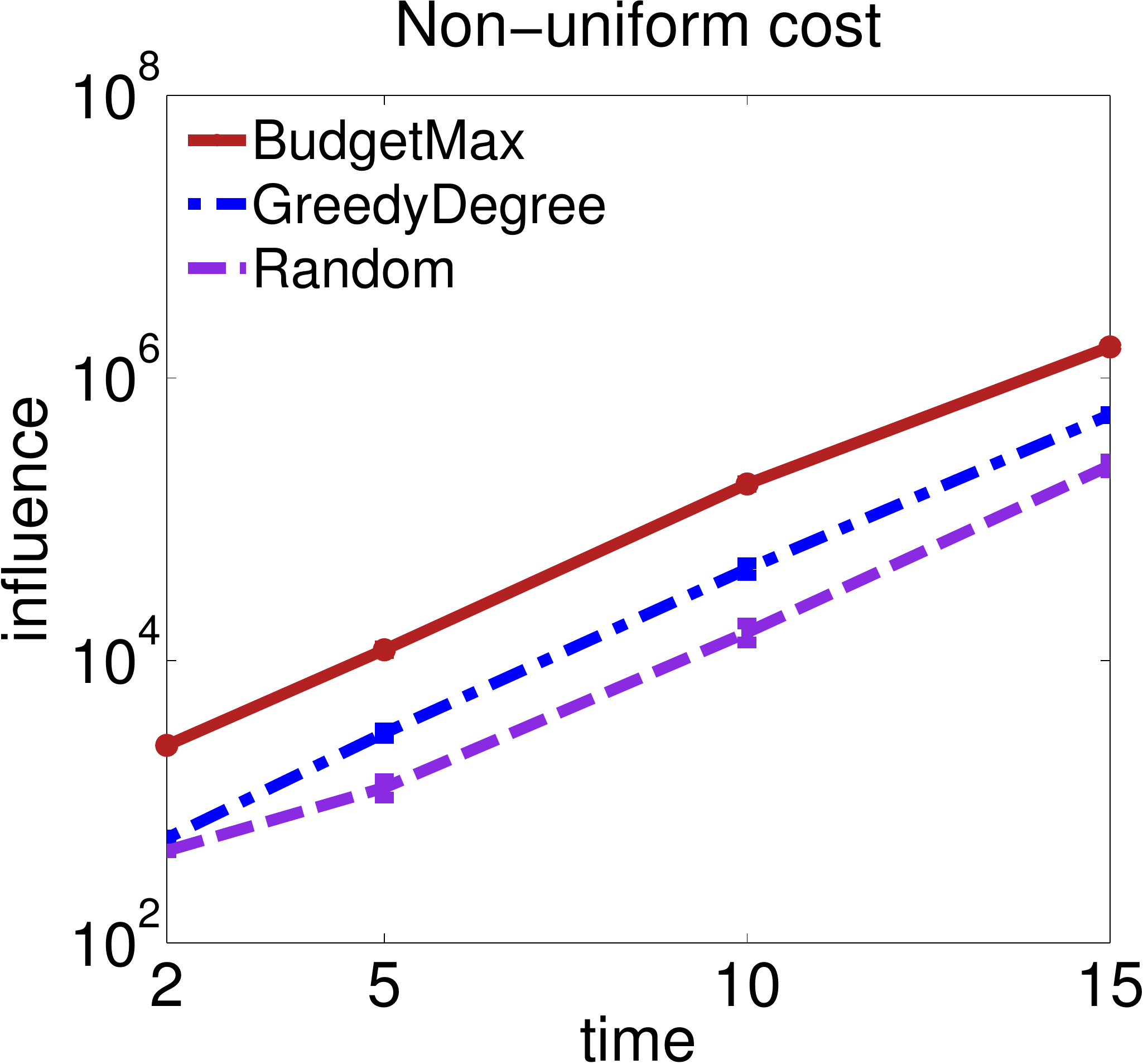} &
\includegraphics[width=0.2\textwidth]{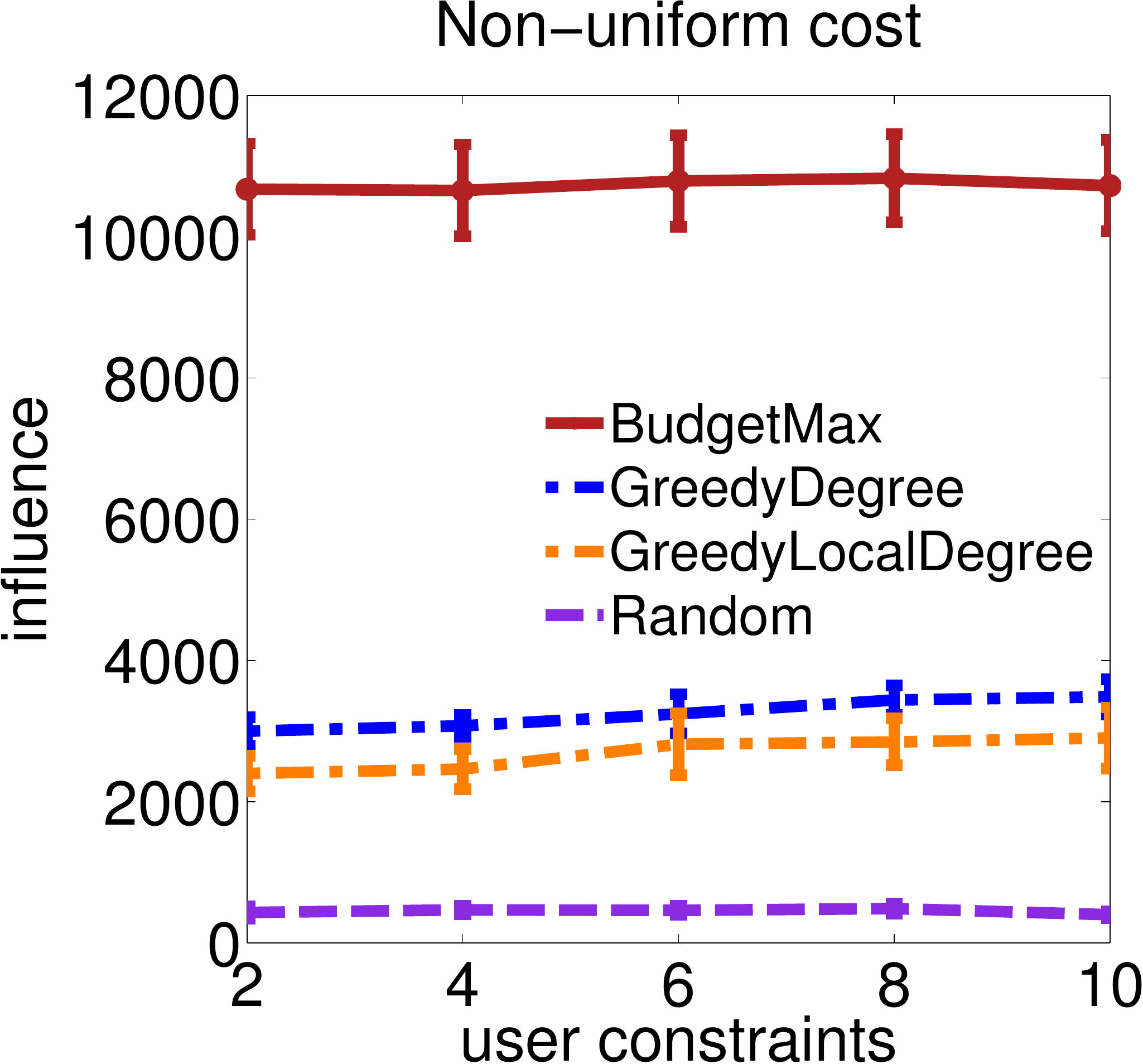} \\
(a) By products & (b) By product budgets & (c) By user constraints & (d) By time & (e) By group limits
\end{tabular}
 \caption{\label{inf-budget-syn} Over the 64 product-specific diffusion networks, each of which has a total 1,048,576 nodes, the estimated influence (a) for increasing the number of products by fixing the product-budget at 1.0 and user-constraint at 2; (b) for increasing product-budget by fixing user-constraint at 2; (c) for increasing user-constraint by fixing product-budget at 1.0; (d) for different time window T; and (e) for increasing user-constraint with group-limit 16 by fixing product-budget at 1.0.}
\end{figure*}

\subsection{Real-world Data}

Finally, we investigate the allocation quality on real-world datasets. The MemeTracker data contains 300 million blog posts and articles collected for the top 5,000 most active media sites from four million websites between March 2011 and February 2012~\citep{RodLesSch13}. The flow of information was traced using quotes which are short textual phrases spreading through the websites. Because all published documents containing a particular quote are time-stamped, a cascade induced by the same quote is a collection of times when the media site first mentioned it. The dataset is divided into groups, each of which consists of cascades built from quotes that were mentioned in posts containing a particular keyword. We have selected 64 groups with at least 100,000 cascades as our products, which include many well-known events such as `apple and jobs', `tsunami earthquake', `william kate marriage', `occupy wall-street', etc.

\noindent{\bf Learning diffusion networks.} On the real-world datasets, we have no prior-knowledge about the diffusion network structure of each meme. The only information we have is the time stamp at which each meme was forwarded in each cascade, so this setting is much more challenging than that of the synthetic experiments. We evenly split the data into the training and testing sets.  On the training set, we first learn each diffusion network by assuming exponential pairwise transmission functions ~\citep{GomBalSch11} for simplicity, although our method can be trivially adapted to the more sophisticated learning algorithms~\citep{DuSonSmoYua12, DuSonWooZha13}. Meanwhile, we also infer the diffusion network structures by fitting the classic discrete-time independent cascade model where the pairwise infection probability is learned based on the method of~\citep{NetPraSanSuj12}, and the step-length is set to one. Then, we can optimize the allocation by running our greedy algorithm over these inferred networks assuming the discrete-time diffusion model. We refer to this implementation as the Greedy(discrete) method. Moreover, because we also have no ground-truth information about cost of each node, we focus on the uniform-cost case, specifically.

\noindent{\bf Influence maximization.} After we find an allocation over the learned networks, we evaluate the performance of the two methods on the held-out testing cascades as follows : given an product-node pair $(i,j)$, let $\Ccal(j)$ denote the set of cascades induced by product $i$ that contains node $j$. The average number of nodes coming after $j$ for all the cascades in $\Ccal(j)$ is treated as the average influence by assigning product $i$ to node $j$. Therefore, the influence of an allocation is just the sum of the average influence of each product-node pair in the solution. Because we have 64 representative products, in order to motivate the competitions to the available allocation slots, we randomly select 128 nodes as our target users. Figure~\ref{inf-real} presents the evaluated results by varying the number of products (a), product constraints (b), user-constraints (c) and the observation window $T$, respectively.  It clearly demonstrates that \budgetmax can find an allocation that indeed induces the largest diffusions contained in the testing data with an average $20$-percent improvement overall.

\noindent{\bf Visualization.} We further plot part of the allocation in Figure~\ref{demo} to get a qualitative intuition about the solution where the red representative memes are assigned to the respective media-sites. For example, `tsunami earthquake' is assigned to `japantoday.com', `wall-street-occupy' is assigned to `finance.yahoo.com', etc. Moreover, because different memes can have diverse diffusion networks with heterogeneous pairwise transmission function, the selected nodes are thus the ones that can invoke faster potential spreading for one or several memes along time, which include a few very popular media sites such as nytimes.com, cnn.com and several modest sites~\citep{BakHofMasWat11} such as freep.com, localnews8.com, etc.

\begin{figure*}[t]
 \centering
 \renewcommand{\tabcolsep}{2pt}
 \begin{tabular}{cccc}
\includegraphics[width=0.22\textwidth]{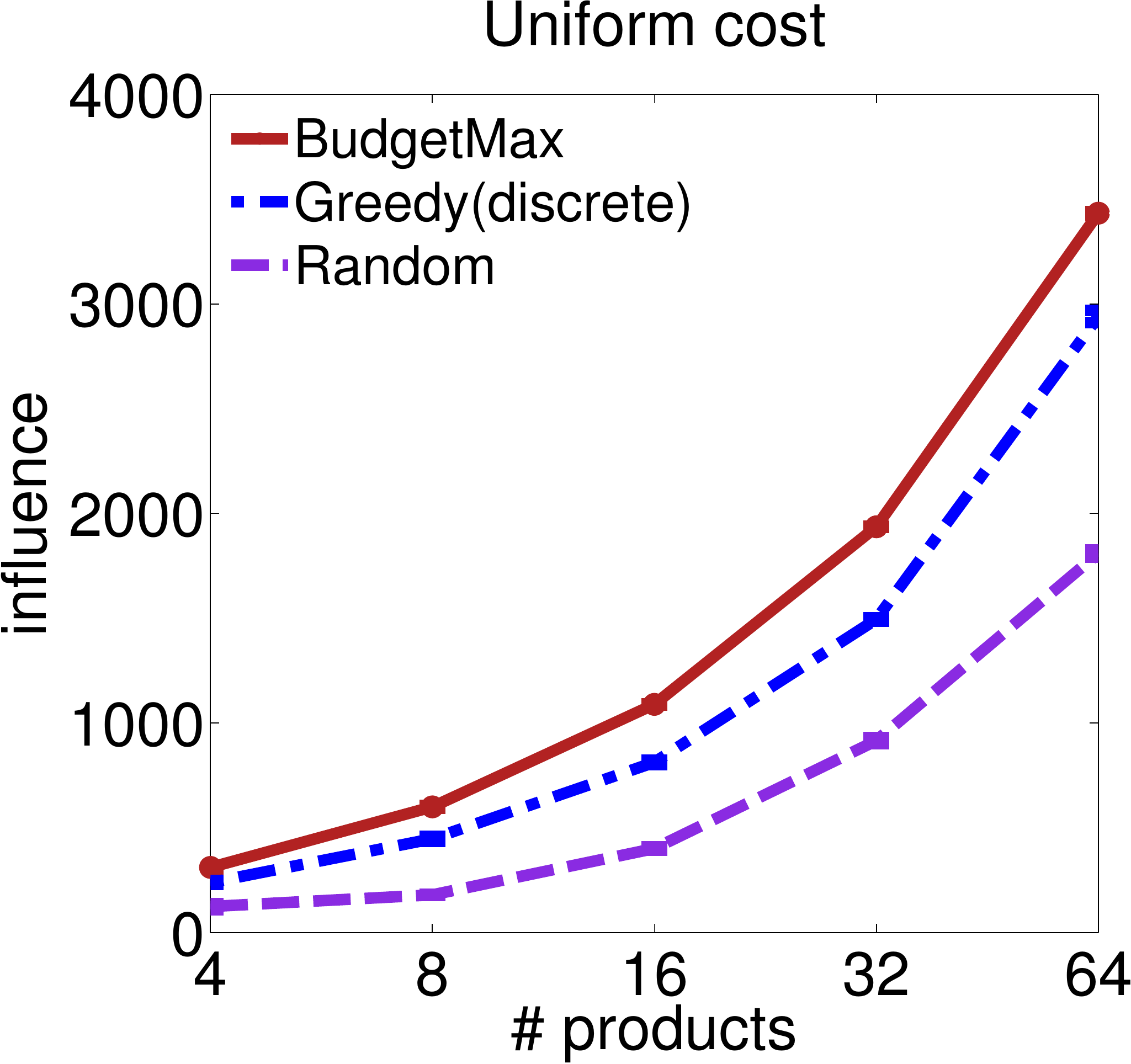}
& \includegraphics[width=0.22\textwidth]{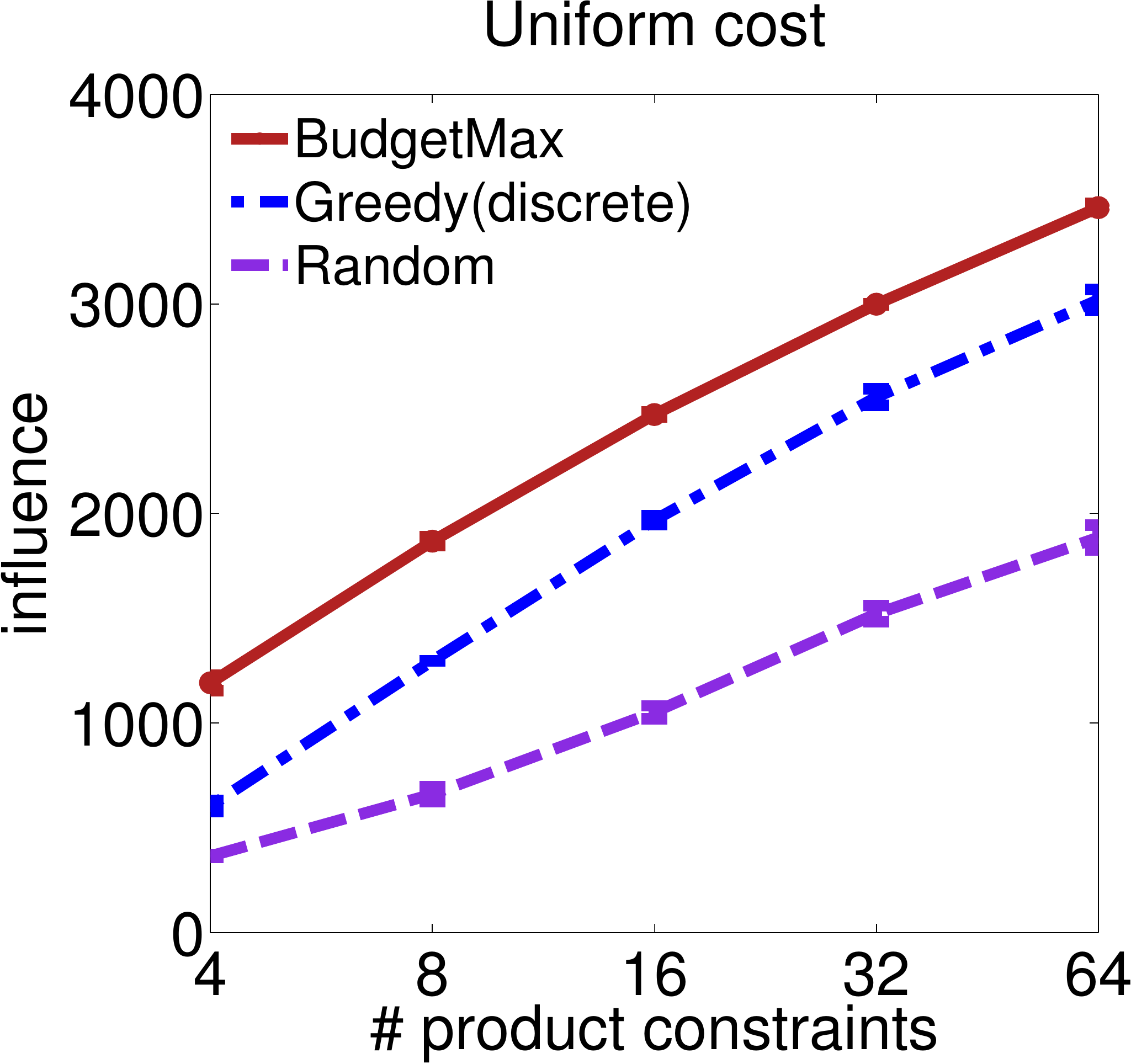}
& \includegraphics[width=0.22\textwidth]{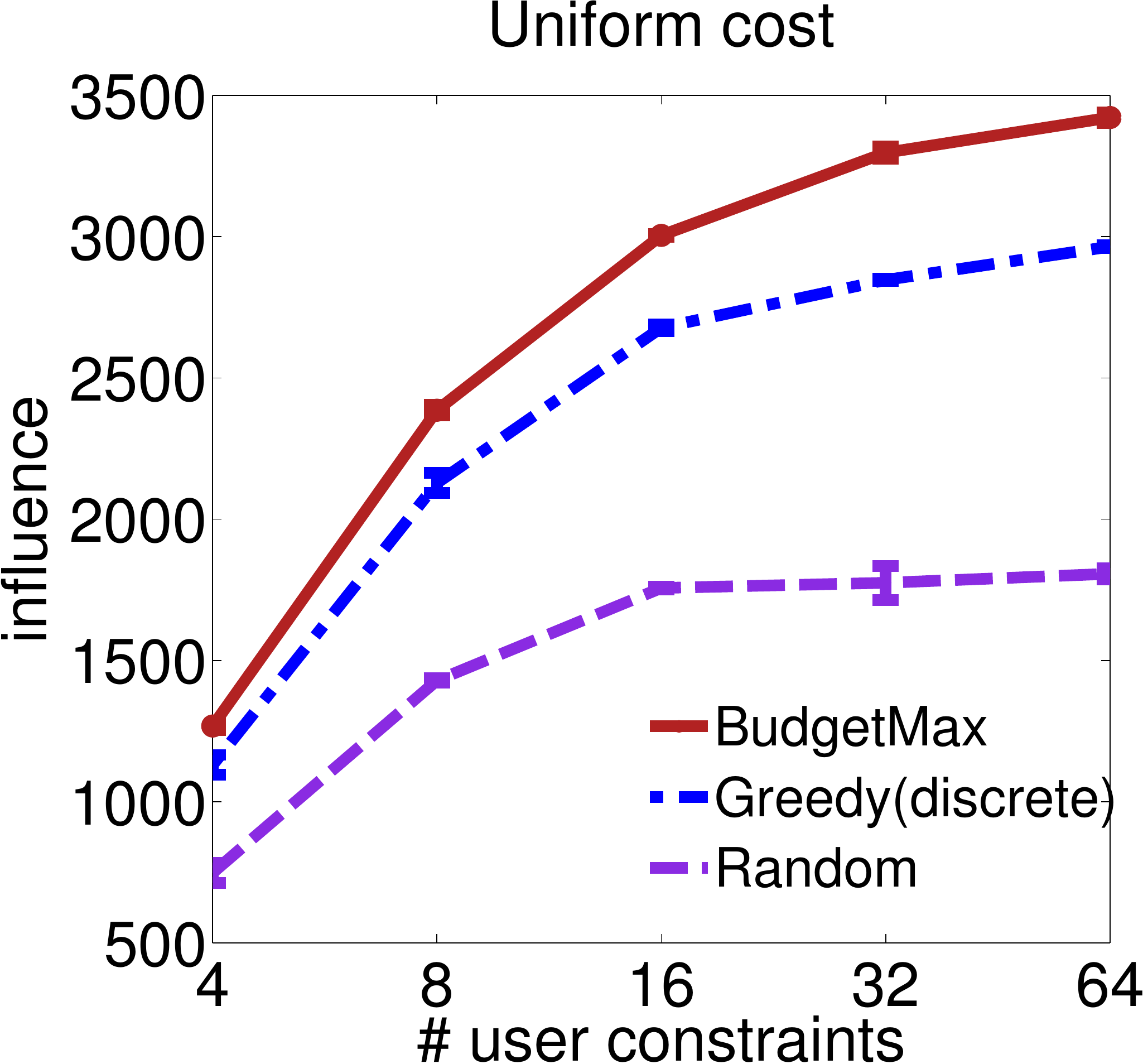}
& \includegraphics[width=0.22\textwidth]{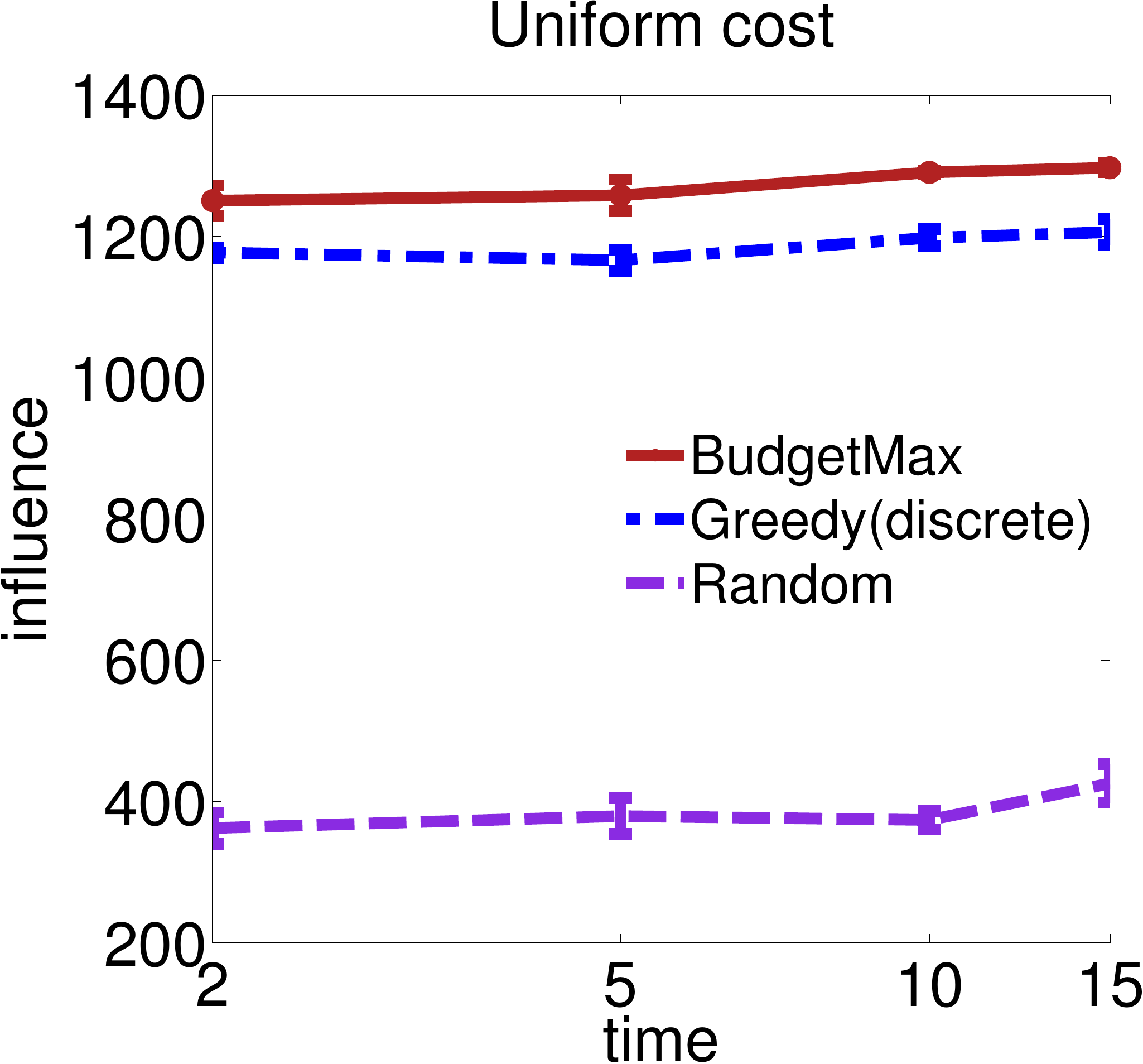} \\
(a) By products & (b)  By product constraints & (c) By user constraints & (d) By time
\end{tabular}
 \caption{\label{inf-real}  Over the inferred 64 product-specific diffusion networks, the true influence estimated from separated testing data (a) for increasing the number of products by fixing the product-constraint at 8 and user-constraint at 2; (b) for increasing product-constraint by fixing user-constraint at 2; (c) for increasing user-constraint by fixing product-constraint at 8; (d) for different time window T.}
\end{figure*}

\begin{figure}[t]
 \centering
 \begin{tabular}{c}
\includegraphics[width=0.4\columnwidth]{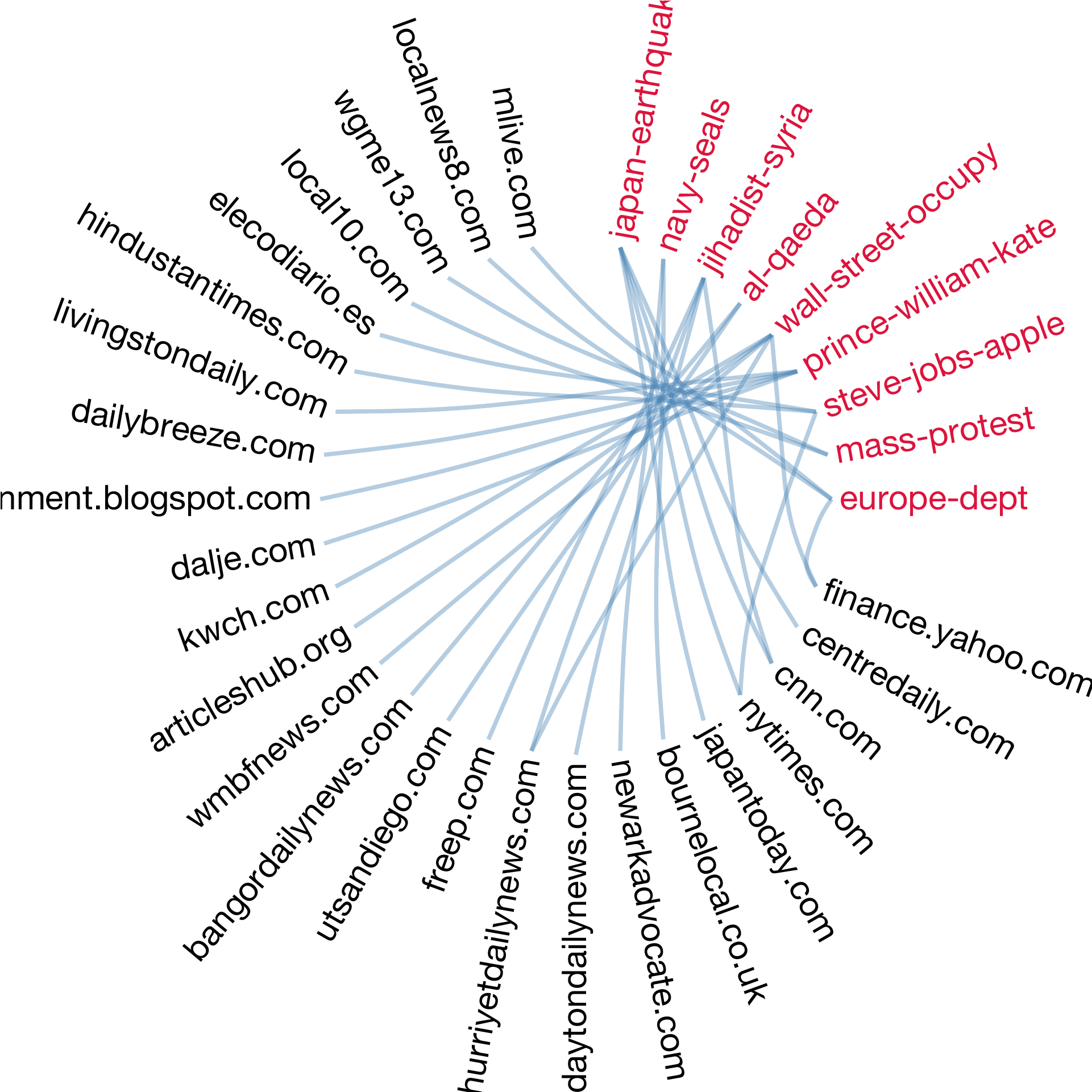}
\end{tabular}
 \caption{\label{demo}
The allocation of memes to media sites.}
\end{figure}

\vspace{-2mm}
\section{Conclusion}
We study the problem of maximizing the influence of multiple types of products  (or information) in realistic continuous-time diffusion networks, subject to various constraints: different products can have different diffusion structures; only influence within given time windows is considered; each user can only be recommended to a small number of products; each product has a limited budget and assigning it to users has costs. We provide a novel formulation as a submodular maximization under an intersection of matroid constraints and group-knapsack constraints, and then design an efficient adaptive threshold greedy algorithm with provable approximation guarantees. Experiment results show that the proposed algorithm performs significantly better than other scalable alternatives in both synthetic and real world datasets. 
\vspace{-2mm}

\bibliographystyle{plainnat}
\bibliography{bibfile}

\appendix

\section{Complete Proofs}

\subsection{Uniform Cost}

We first prove that a theorem for Problem~\ref{pro:infMax} with general normalized monotonic submodular function $f(S)$ and general $P$ (Theorem~\ref{thm:dtgreedy}) and $k=0$, and then specify the guarantee for our influence maximization problem (Theorem~\ref{thm:infMax_uni}).

Suppose $G=\cbr{g_1,\dots, g_{|G|}}$ in the order of selection, and let $G^t =\cbr{\g_1, \dots, \g_t}$.
Let $C_t$ denote all those elements in $O \setminus G$ that satisfy the following:
they are still feasible before selecting the $t$-th element $g_t$ but are infeasible after selecting $g_t$.
Formally, 
$$
	C_t=\cbr{z \in O \setminus G: \cbr{z} \cup G^{t-1} \in \Fcal, \cbr{z} \cup G^t \not\in \Fcal}.
$$

In the following, we will prove three claims and then use them to prove the theorems.
Recall that for any $i\in \Ground$ and $S \subseteq \Ground$, the marginal gain of $z$ with respect to $S$ is denoted as
$$
	f(z| S) := f(S\cup\cbr{z}) - f(S)
$$
and its approximation is denoted by 
$$
	\widehat f(z| S) = \widehat  f(S\cup\cbr{z}) - \widehat f(S).
$$
When $|f(S)- \widehat{f}(S)| \leqslant \epsilon$ for any $S \subseteq \Ground$, we have 
$$
	|\widehat f(z| S)  - f(z| S) | \leqslant 2\epsilon
$$ for any $z\in \Ground$ and $S \subseteq \Ground$.

\smallskip
\noindent
\textbf{Claim~\ref{cla:size}. }
{\it
$\sum_{i=1}^t |C_i| \leqslant P t$, for $t =1, \dots, |G|$.
}

\begin{proof}
We first show the following property about matroids: for any $Q \subseteq \Ground$, the sizes of any two maximal independent subsets $T_1$ and $T_2$ of $Q$
can only differ by a multiplicative factor at most $P$. 
Here, $T$ is a maximal independent subset of $Q$ if and only if:
\begin{itemize}[itemsep=0pt]
\item $T \subseteq Q$;
\item $T \in \Fcal = \bigcap_{i=1}^P \Ical_p$;
\item $T \cup \cbr{z} \not \in \Fcal$ for any $z \in Q \setminus T$.
\end{itemize}

To prove the property, note that for any element $ z \in T_1 \setminus T_2$,
$\cbr{z} \cup T_2$ violates at least one of the matroid constraints since $T_2$ is maximal.
Let $V_i (1 \leqslant i \leqslant P)$ denote all elements in $T_1 \setminus T_2$ that violates the $i$-th matroid,
and then partition $T_1 \cap T_2$ arbitrarily among these $V_i$'s so that they cover $T_1$.
Note that the size of each $V_i$ must be at most that of $T_2$,
since otherwise by the Exchange axiom, there would exist $z \in V_i \setminus T_2$ that
can be added to $T_2$ without violating the $i$-th matroid, which is contradictory to the construction.
Therefore, the size of $T_1$ is at most $P$ times that of $|T_2|$.

Now we apply the property to prove the claim.
let $Q$ be the union of $G^{t}$ and $\bigcup_{i=1}^t C_t$.
On one hand, $G^{t}$ is a maximal independent subset of $Q$, since no element in $\bigcup_{i=1}^t C_t$ can be added to $G^t$ without violating the matroid constraints.
On the other hand, $\bigcup_{i=1}^t C_t$ is an independent subset of $Q$, since it is part of the optimal solution.
Therefore, $\bigcup_{i=1}^t C_t$ has size at most $P$ times $|G^t|$, which is $Pt$.
\end{proof}

\smallskip
\noindent
\textbf{Claim~\ref{cla:gain}. }
{\it
Suppose $g_t$ is selected at the threshold $\tau_t$. 
$f(j|G^{t-1}) \leqslant (1+\delta) \tau_t + 4\epsilon + \frac{\delta}{\nGround} f(G), \forall j \in C_t$.
}

\begin{proof}
First, consider $\tau_t > w_{L+1} = 0$.
We clearly have $\widehat f( \g_t | \Greedy^{t-1}) \geqslant \tau_t$ and thus $f( \g_t | \Greedy^{t-1}) \geqslant \tau_t - 2\epsilon$.
For each $j \in C_t$, if $j$ were considered at a stage earlier,
than it would have been added to $\Greedy$ since adding it to $\Greedy^{t-1}$ will not violate the constraint.
However, $j \not\in \Greedy^{t-1}$, so $\widehat f( j| \Greedy^{t-1}) \leqslant (1+\delta) \tau_t$.
Then 
$$
	f(j|\Greedy^{t-1}) \leqslant (1+\delta) \tau_t + 2\epsilon.
$$	

Next, consider $\tau_t = w_{L+1} = 0$.
For each $j \in C_t$, we have $\widehat f(j|\Greedy) < \frac{\delta }{\nGround} d$.
In fact, by greedy selection we have the first element $g_1$ is of value $\widehat f(g_1) = d$, so $d \leqslant f(g_1) + \epsilon$. Then
$$
	f(j|\Greedy) < \frac{\delta }{\nGround} f(G) + 4\epsilon.
$$
The claim follows by combining the two cases.
\end{proof}

\smallskip
\noindent
\textbf{Claim~\ref{cla:com}. }
{\it
The marginal gain of $O\setminus G$ satisfies
$$\sum_{j \in O\setminus G} f(j|\Greedy) \leqslant [(1+\delta) P + \delta] f(G)  + (6+2\delta) \epsilon P |G|. $$
}

\begin{proof}
Combining Claim~\ref{cla:size} and Claim~\ref{cla:gain}, we have
\begin{align*}
 \sum_{j \in O\setminus G} f(j|\Greedy) = \sum_{t=1}^{|G|} \sum_{j \in C_t} f(j|\Greedy) 
& \leqslant (1+\delta) \sum_{t=1}^{|G|} |C_t| \tau_t + \delta f(G)  + 4\epsilon \sum_{t=1}^{|G|} |C_t|\\
& \leqslant (1+\delta) \sum_{t=1}^{|G|} |C_t| \tau_t + \delta f(G)  + 4\epsilon P |G|.
\end{align*}
The term $\sum_{t=1}^{|G|} |C_t| \tau_t \leqslant P \sum_{t=1}^{|G|} \tau_t$ by Claim~\ref{cla:size} and a technical lemma (Lemma~\ref{lem:seqsum}).
The claim follows from the fact that $f(G) = \sum_t f(g_t|G^{t-1}) \geqslant \sum_t (\tau_t - 2\epsilon)$.
\end{proof}

\begin{lemma}\label{lem:seqsum}
If $\sum_{i=1}^t \sigma_{i-1} \leqslant t$ for $t=1,\dots, K$ and $\rho_{i-1} \geqslant \rho_i$ for $i=1,\dots,K-1$ with $\rho_i, \sigma_i\geqslant 0$,
then $\sum_{i=1}^K \rho_i\sigma_i \leqslant \sum_{i=1}^K \rho_{i-1}$.
\end{lemma}
\begin{proof}
Consider the linear program
\begin{eqnarray*}
V&=&\max_{\sigma} \sum_{i=1}^K \rho_i\sigma_i\\
& \textrm{s.t.} & \sum_{i=1}^t \sigma_{i-1} \leqslant t, \ \ t=1,\dots,K,\\
&& \sigma_i\geqslant 0, \ \ i=1,\dots,K-1
\end{eqnarray*}
with dual
\begin{eqnarray*}
W&=&\min_{u} \sum_{i=1}^K t u_{t-1}\\
& \textrm{s.t.} & \sum_{t=i}^{K-1} u_t \geqslant \rho_i, \ \ i=0,\dots,K-1,\\
&& u_t\geqslant 0, \ \ t=0,\dots,K-1.
\end{eqnarray*}
As $\rho_i \geqslant \rho_{i+1}$, the solution $u_i = \rho_i - \rho_{i+1},i=0,\dots,K-1$ (where $\rho_K=0$)
is dual feasible with value $\sum_{t=1}^K t (\rho_{t-1}-\rho_t) = \sum_{i=1}^{K} \rho_{i-1}$.
By weak linear programming duality,  $\sum_{i=1}^K \rho_i\sigma_i \leqslant V \leqslant W \leqslant \sum_{i=1}^K \rho_{i-1}$.
\end{proof}

\begin{theorem}\label{thm:dtgreedy}
For Problem~\ref{pro:infMax} with $k=0$, suppose Algorithm~\ref{alg:greedyFixedDensity} uses $\rho=0$ and $\widehat f$ to estimate the function $f$
which satisfies $|\widehat f(S) - f(S)| \leqslant \epsilon$ for all $S \subseteq \Ground$.
Then it uses $\Ocal(\frac{\nGround}{\delta}\log\frac{\nGround}{\delta})$ evaluations of $\widehat f$,
and returns a greedy solution $\Greedy$ with
$$f(\Greedy) \geqslant \frac{1}{(1+2\delta)(P+1)} f(\Optimal) - \frac{4 P |\Greedy|}{P+\cur_f}\epsilon$$
where $O$ is the optimal solution.
\end{theorem}

\begin{proof}
By submodulairty and Claim~\ref{cla:com}, we have 
\begin{align*}
f(O) & \leqslant f(O\cup G) \leqslant f(G) + \sum_{j \in O\setminus G} f(j|\Greedy) \leqslant (1+\delta)(P + 1) f(G)  + (6+2\delta) \epsilon P |G|
\end{align*}
which leads to the bound in the theorem.

The number of evaluations is bounded by $\Ocal(\frac{\nGround}{\delta}\log\frac{\nGround}{\delta})$
since there are $\Ocal(\frac{1}{\delta}\log \frac{\nGround}{\delta})$ thresholds,
and there are $\Ocal(\nGround)$ evaluations at each threshold.
\end{proof}

Theorem~\ref{thm:dtgreedy} essentially shows $f(G)$ is close to $f(O)$ up to a factor roughly $(1+P)$, which then leads to the following guarantee for our influence maximization problem.Suppose product $i \in \Item$ spreads according to diffusion network $\Graph_i = (\Node, \Edge_i)$, and let $i^*=\argmax_{i\in\Item}|\Edge_i|$.

\smallskip
\noindent
\textbf{Theorem~\ref{thm:infMax_uni}. }
{\it
For influence maximization with uniform cost,  Algorithm~\ref{alg:greedyFixedDensity} (with $\rho=0$) outputs a solution $G$ with
$
	f(\Greedy) \geqslant \frac{1-2\delta}{3} f(\Optimal)
$
in expected time $\widetilde\Ocal\left(\frac{|\Edge_{i^*}|+|\Node|}{\delta^2}  + \frac{|\Item||\Node|}{\delta^3} \right).$
}

\begin{proof}
In the influence maximization problem, the number of matroids is $P=2$.
Also note that $|G| \leqslant f(G) \leqslant f(O)$, which leads to $4 |\Greedy|\epsilon \leqslant 4\epsilon f(O)$.
The approximation guarantee then follows from setting $\epsilon\leqslant \delta/16$ when using \continmax~\citep{DuSonZhaGom13} to estimate the influence.

The runtime is bounded as follows.
In Algorithm~\ref{alg:greedyFixedDensity}, we need to estimate the marginal gain of adding one more product to the current solution.
In \continmax~\citep{DuSonZhaGom13}, building the initial data structure takes time
$$
    \Ocal\left((|\Edge_{i^*}|\log|\Node| + |\Node|\log^2 |\Node|) \frac{1}{\delta^2} \log \frac{|\Node|}{\delta} \right)
$$
and afterwards each function evaluation takes time $$
    \Ocal\left(\frac{1}{\delta^2} \log \frac{|\Node|}{\delta} \log\log|\Node|\right).
$$
As there are $\Ocal\left(\frac{\nGround}{\delta}\log\frac{\nGround}{\delta}\right)$
evaluations where $N=|\Item||\Node|$, the runtime of our algorithm follows.
\end{proof}

\subsection{Non-uniform cost}

We first prove that a theorem for Problem~\ref{pro:infMax} with general normalized monotonic submodular function $f(S)$ and general $P$ (Theorem~\ref{thm:densityEnu}), and then specify the guarantee for our influence maximization problem (Theorem~\ref{thm:infMax}).

\begin{theorem}\label{thm:densityEnu}
Suppose Algorithm~\ref{alg:densityEnu} uses $\widehat f$ to estimate the function $f$
which satisfies $|\widehat f(S) - f(S)| \leqslant \epsilon$ for all $S \subseteq \Ground$.
There exists a $\rho$ such that
\begin{align*}
f(S_\rho) \geqslant \frac{\max\cbr{1,|A_\rho|} }{(P+2k+1)(1+2\delta)} f(O) - 8\epsilon |S_\rho|
\end{align*}
where $A_\rho$ is the set of active knapsack constraints.
\end{theorem}

\begin{proof}
Consider the optimal solution $O$ and set $\rho^* = \frac{2}{P+ 2k+1} f(O)$.
By submodularity, we have $d \leqslant f(O) \leqslant |\Ground| d$, so $\rho \in \sbr{ \frac{2d}{P+2k+1}, \frac{2|\Ground|d}{P+2k+1} }$,
and there is a run of Algorithm~\ref{alg:greedyFixedDensity} with $\rho$ such that $\rho^* \in [\rho, (1+\delta) \rho]$.
In the following we consider this run.

\smallskip
\noindent
{\bf Case 1} Suppose $|A_\rho| = 0 $. The key observation in this case is that since no knapsack constraints are active,
the algorithm runs as if there were only matroid constraints. Then the argument for matroid constraints can be applied.
More precisely, let
$$
    O_+ := \cbr{z \in O\setminus S_\rho :  f(z|S_\rho) \geqslant c(z)\rho + 2\epsilon}
$$
$$
    O_- := \cbr{z \in O\setminus S_\rho : z \not\in O_+ }.
$$
Note that all elements in $O_+$ are feasible.
Following the argument of Claim~\ref{cla:com} in Theorem~\ref{thm:dtgreedy}, we have
\begin{align}
    f(O_+|S_\rho) \leqslant ((1+\delta)P + \delta) f(S_\rho) + (4+2\delta) \epsilon P|S_\rho|.
\label{eqn:matroidcase1}
\end{align}
Also, by definition the marginal gain of $O_-$ is:
\begin{align}
    f(O_-|S_\rho) \leqslant k \rho + 2\epsilon |O_-| \leqslant k \rho + 2\epsilon P |S_\rho|
\label{eqn:matroidcase2}
\end{align}
where the last inequality follows from the fact that $S_\rho$ is a maximal independent subset
and $O_-$ is an independent subset of $O\cup S_\rho$, and the fact that the sizes of any two maximal independent subsets
in the intersection of $P$ matroids can differ by a factor of at most $P$.
Plugging (\ref{eqn:matroidcase1})(\ref{eqn:matroidcase2}) into $f(O) \leqslant f(O_+ |S_\rho) + f(O_- |S_\rho) + f(S_\rho)$ we obtain the bound
$$
    f(S_\rho) \geqslant \frac{f(O)}{(P+ 2 k + 1)(1+\delta)}  - \frac{(6+2\delta) \epsilon P|S_\rho| }{(P+ 1)(1+\delta)}.
$$

\smallskip
\noindent
{\bf Case 2}
Suppose $|A_\rho| > 0 $.
For any $i \in A_\rho$ (\ie, the $i$-th knapsack constraint is active), consider the step when $i$ is added to $A_\rho$.
Let $\Greedy_i = \Greedy \cap \Ground_{i*}$, and we have $c(\Greedy_i) + c(z) > 1$.
Since every element $g$ we include in $\Greedy_i$ satisfies $\widehat{f}(g|\Greedy)  \geqslant c(g) \rho$ with respect to the solution $\Greedy_i$ when $g$ is added.
Then $f(g|\Greedy) = f_i(g|\Greedy_i) \geqslant c(g) \rho - 2\epsilon$, and
we have
\begin{align}
f_i(\Greedy_i \cup \cbr{z}) & \geqslant \rho [c(\Greedy_i) + c(z)] - 2 \epsilon (|G_i| + 1) > \rho - 2 \epsilon (|G_i| + 1).\label{eqn:knapcase1}
\end{align}
Note that $\Greedy_i$ is non-empty since otherwise the knapsack constraint will not be active.
Any element in $\Greedy_i$ is selected before or at $w_{t}$, so $f_i(G_i) \geqslant w_{t} - 2\epsilon$.
Also, note that $z$ is not selected in previous thresholds before $w_{t}$, so $f_i( \cbr{z} | \Greedy_i) \leqslant (1+\delta) w_t + 2\epsilon$ and thus
\begin{align}
f_i( \cbr{z} | \Greedy_i) \leqslant (1+\delta) f_i(\Greedy_i) + 2\epsilon (2+\delta). \label{eqn:knapcase2}
\end{align}
Plugging (\ref{eqn:knapcase1})(\ref{eqn:knapcase2}) into $ f_i(\Greedy_i \cup \cbr{z}) = f_i(\Greedy_i) + f_i( \cbr{z}  | \Greedy_i)$ leads to
\begin{align*}
f_i(\Greedy_i)  \geqslant \frac{\rho}{(2+\delta)} - \frac{2\epsilon(|G_i| + 3 + \delta) }{(2+\delta)}  & \geqslant \frac{1}{2(1+2\delta)} \rho^* - \frac{2\epsilon(|G_i| + 3 + \delta) }{(2+\delta)} \\
& \geqslant \frac{f(O)}{(P+2k+1)(1+2\delta)} - 5 \epsilon |G_i|.
\end{align*}
Summing up over all $i \in A_\rho$ leads to the desired bound.
\end{proof}

Suppose item $i \in \Item$ spreads according to
the diffusion network $\Graph_i = (\Node, \Edge_i)$. Let $i^*=\argmax_{i\in\Item}|\Edge_i|$.
By setting $\epsilon=\delta/16$ in Theorem~\ref{thm:densityEnu}, we have:

\smallskip
\noindent
\textbf{Theorem~\ref{thm:infMax}. }
{\it
In Algorithm~\ref{alg:densityEnu}, there exists a $\rho$ such that
$$
    f(S_\rho) \geqslant  \frac{\max\cbr{k_a, 1} }{(2|\Item|+2) (1+3\delta)} f(O)
$$
where $k_a$ is the number of active knapsack constraints.
The expected running time is $\widetilde\Ocal\left(\frac{|\Edge_{i^*}|+|\Node|}{\delta^2}  + \frac{|\Item||\Node|}{\delta^4} \right).$
}

\end{document}